\pdfoutput=1

\documentclass[11pt]{article}

\usepackage{EMNLP2022}

\usepackage{times}
\usepackage{latexsym}

\usepackage[T1]{fontenc}

\usepackage[utf8]{inputenc}

\usepackage{microtype}

\usepackage{inconsolata}

\usepackage{tikz}

\usepackage{amsmath}
\usepackage{amssymb}
\usepackage{amsthm}         %
\usepackage{caption}
\usepackage{graphicx}
\usepackage{color}
\usepackage{subcaption}
\usepackage{tabularx}
\usepackage{float}
\usepackage{mathtools}
\usepackage{thmtools,thm-restate}
\usepackage{paralist} 
\usepackage{multirow}
\usepackage{natbib}
\usepackage{framed}

\theoremstyle{definition}
\newtheorem{definition}{Definition}[section]

\newtheorem{lemma}{Lemma}[section]
\newtheorem{assumption}{Assumption}[section]
\newtheorem{proposition}{Proposition}[section]

\newcommand{\ie}{\textit{i}.\textit{e}.}
\newcommand{\eg}{\textit{e}.\textit{g}.}
\newcommand{\xspace}{\mathcal{X}}
\newcommand{\yspace}{\mathcal{Y}}
\newcommand{\aspace}{\mathcal{A}}
\newcommand{\zspace}{\mathcal{Z}}

\newcommand{\dkl}{D_{\text{KL}}}

\title{Conditional Supervised Contrastive Learning for Fair Text Classification}

\author{Jianfeng Chi \\ University of Virginia \\ \texttt{jc6ub@virginia.edu} \And
        William Shand \\ University of Virginia \\ \texttt{wss2ec@virginia.edu} \And
        Yaodong Yu \\ UC Berkeley \\ \texttt{yyu@eecs.berkeley.edu}
        \AND
        Kai-Wei Chang \\ UCLA \\ \texttt{kwchang@cs.ucla.edu} \And
        Han Zhao \\ UIUC \\ \texttt{hanzhao@illinois.edu}
        \And
        Yuan Tian \\ UCLA \\ \texttt{yuant@ucla.edu}
}

\begin{document}
\maketitle

\begin{abstract}
Contrastive representation learning has gained much attention due to its superior performance in learning representations from both image and sequential data. However, the learned representations could potentially lead to performance disparities in downstream tasks, such as increased silencing of underrepresented groups in toxicity comment classification. In light of this challenge, in this work, we study learning fair representations that satisfy a notion of fairness known as equalized odds for text classification via contrastive learning. Specifically, we first theoretically analyze the connections between learning representations with a fairness constraint and \emph{conditional supervised contrastive objectives}, and then propose to use conditional supervised contrastive objectives to learn fair representations for text classification. We conduct experiments on two text datasets to demonstrate the effectiveness of our approaches in balancing the trade-offs between task performance and bias mitigation among existing baselines for text classification. Furthermore, we also show that the proposed methods are stable in different hyperparameter settings.\footnote{Our code is publicly available at: \url{https://github.com/JFChi/CSCL4FTC}.}
\end{abstract}

\section{Introduction}
Recent progress in natural language processing (NLP) has led to its increasing use in various domains, such as machine translation, virtual assistants, and social media monitoring.
However, studies have demonstrated societal bias in existing NLP models~\citep{bolukbasi2016man,zhao-etal-2017-men,may2019measuring, bordia2019identifying, hutchinson2020social, webster2020measuring, de2021stereotype,sheng-etal-2021-societal}. In one major NLP application, text classification, bias is referred as the performance disparity of the trained classifiers over different demographic groups such as gender and ethnicity~\citep{sun-etal-2019-mitigating,weidinger2021ethical}. 
Such bias poses potential risks: for example, toxicity classification models in online social media platforms show disparate performance in different social groups, leading to increased silencing of under-served groups~\citep{dixon2018measuring, blodgett2020language}.

Meanwhile, an increasing line of work in contrastive learning (CL) has led to significant advances in representation learning~\citep{hadsell2006dimensionality, logeswaran2018efficient, he2020momentum, henaff2020data, pmlr-v119-chen20j, khosla2020supervised, gao2021simcse}. 
The general idea of contrastive learning in these works is to learn representations such that similar examples stay close to each other while dissimilar ones are far apart. 
Inspired by those works, recent works~\citep{shen2021contrastive, tsai2021conditional, tsai2022conditional} also propose to leverage contrastive learning to learn fair representations in classification. However, these works either lack theoretical justifications for the proposed approaches or adopt \emph{demographic parity}~\citep{dwork2012fairness} as the fairness criterion, which eliminates the perfect classifier in the common scenario when the \emph{base rates} differ among demographic groups~\citep{hardt2016equality, han2019inherent}.

In this work, we aim to mitigate bias in text classification models via contrastive learning. In particular, we adopt the fairness notion, \emph{equalized odds} (EO)~\citep{hardt2016equality}, which asks for equal true positive rates (TPRs) and false positive rates (FPRs) across different demographic groups~\citep{zhao2019conditional}. Based on information-theoretic concepts, we bridge the problem of learning fair representations with equalized odds constraint with contrastive learning objectives. We then propose an algorithm, called \emph{conditional supervised contrastive learning}, to learn fair text classifiers. 

Empirically, we conduct experiments on two text classification datasets (\eg, toxic comment classification and biography classification) to show the proposed methods (1) can flexibly tune the trade-offs between main task performance and the fairness constraint; (2) achieve the best trade-offs between main task performance and equalized odds compared to the existing bias mitigation approaches in text classification;  
(3) are stable to different hyperparameter settings, such as data augmentations, temperatures, and batch sizes. 
To the best of our knowledge, our work is the first to both theoretically and empirically study how to ensure the EO constraint via contrastive learning in text classification. 

\section{Background}
\label{sec:background}
We use $X\in\xspace$ and $Y\in\yspace$ to denote the random variables for the input text and the categorical label for the main task, respectively.
Furthermore, $A\in\aspace$ is the sensitive attribute (protected group) associated with the input text $X$ (\eg, the gender information in the occupation classification task). The corresponding lowercase letters denote the instantiation of the random variables.
Given a text encoder $f: \xspace \to \zspace$ (\eg, BERT~\citep{devlin-etal-2019-bert}) and a classifier $g: \zspace \to \yspace$, we first transform the input text $X$ into latent representation $Z$ via $f$, and $Z$ is used to give a prediction $\hat{Y}$ via $g$ (\ie, $X \xrightarrow{f} Z \xrightarrow{g} \hat{Y}$). 

In the context of contrastive learning, data augmentation strategies have been widely adopted.
Let $\mathcal{T}$ be a set of data augmentations and $X'$ be the augmented input given the data augmentation $t(\cdot)$: $X'= t(X),~t\sim \mathcal{T}$, where we assume that the augmentation $t$ is sampled uniformly at random from $\mathcal{T}$. Similarly, we have $X' \xrightarrow{f} Z' \xrightarrow{g} \hat{Y}'$. Let $H$ denote the entropy and $I$ denote the mutual information, \eg, $H(Z\mid Z', Y)$ is the conditional entropy of $Z$ given $Z'$ and $Y$, and $I(Z'; Z\mid Y)$ is the conditional mutual information of $Z'$ and $Z$ given $Y$. 
Due to the space limit, we refer readers to~\citet{cover1999elements} for more background knowledge of the related notions (entropy and mutual information) in information theory. 

We assume there is a joint distribution over $X$, $Y$, and $A$ from which the data are sampled. 
Figure~\ref{fig:graphical_model} shows the graphical model of the dependencies between input variables and outputs. 
We also assume that the sensitive attribute $A$ is available only during model training, but it is not available during the testing phase. As a result, any post-processing methods that leverage sensitive attributes for bias mitigation during the testing phase are not feasible in our setting. In this work, we use equalized odds, a more refined fairness criterion for classification problems.

\begin{definition}[Equalized Odds~\citep{hardt2016equality}]
A model satisfies equalized odds if $\hat{Y} \perp A \mid Y$.
\end{definition}

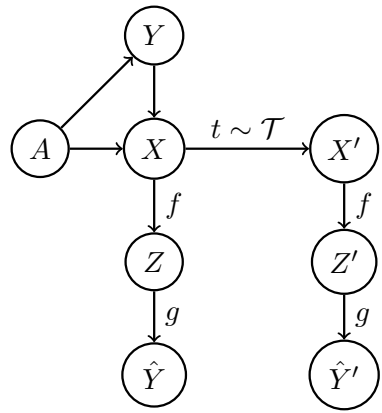
\begin{figure}[t!]
\centering
\begin{tikzpicture}[node distance = {15mm}, main/.style = {draw,circle},line width=0.3mm,->]
\node[main] (Y) {$Y$};
\node[main] (X) [below of=Y] {$X$};
\draw (Y) -- (X);

\node[main] (A) [left of=X] {$A$};
\draw (A) -- (Y);
\draw (A) -- (X);

\node[main] (X') [right of=X, xshift=1cm] {$X'$};
\draw (X) -- node [anchor=south] {$t \sim \mathcal{T}$} (X');

\node[main] (Z) [below of=X] {$Z$};
\draw (X) -- node [anchor=west] {$f$} (Z);

\node[main] (Yhat) [below of=Z] {$\hat{Y}$};
\draw (Z) -- node [anchor=west] {$g$} (Yhat);

\node[main] (Z') [below of=X'] {$Z'$};
\draw (X') -- node [anchor=west] {$f$} (Z');

\node[main] (Yhat') [below of=Z'] {$\hat{Y}'$};
\draw (Z') -- node [anchor=west] {$g$} (Yhat');
\end{tikzpicture}
\caption{
Graphical model of the dependencies between input variables and outputs. Note that we only assume there is a joint distribution over $X$, $Y$, and $A$ from which the data are sampled, so the figure only shows one case of the dependencies over $X$, $Y$, and $A$. 
}
\label{fig:graphical_model}
\end{figure}

At a high level, EO asks the model prediction to be independent of the sensitive attribute conditioned on the task label. If a model perfectly satisfies equalized odds, the differences of true positive rates and false positive rates across demographic groups will be 0. 
Equivalently, it also implies $I(\hat{Y}; A \mid Y)=0$. Consider online comment toxicity classification as a real-world example to motivate the use of EO as a notion of fairness. In this case, false positive cases (benign text comments marked as toxic) can be seen as unintentional censoring, and false negative cases (toxic text comments marked as benign) might result in debates and discomforts~\citep{baldini2021your}. 

In contrast to another well-known group fairness definition, \ie, demographic parity, EO does not require positive prediction rates to be the same across different demographic groups, which could possibly severely downgrade the model performance when the sensitive attribute is correlated to the task label~\citep{hardt2016equality, han2019inherent}.

\section{Our Method}
In this section, we first theoretically connect learning fair representations with contrastive learning (Sec.~\ref{subsec:theory}). In particular, we first show that learning fair representations for equalized odds requires the minimization of $I(Z'; Z \mid Y)$ and the simultaneous maximization of $I(Z'; Z \mid A, Y)$. To this end, we provide an upper bound of $I(Z'; Z \mid Y)$ and a lower bound of $I(Z'; Z \mid A, Y)$ to relax the original objective and then establish a relationship between the bounds and the (conditional) supervised contrastive learning objectives. 
Finally, inspired by our theoretical analysis, we design two practical methods for learning fair representations (Sec.~\ref{subsec:methodology:practical-implementation}). 
Due to the space limit, we defer all detailed proofs to Appendix~\ref{app:proof}.

\subsection{Connections between Contrastive Learning and Learning Fair Representations}
\label{subsec:theory}

In order to learn a model (text encoder followed by classifier) to satisfy equalized odds, we aim to learn a latent representation $Z$ such that $Z \perp A \mid Y$. From an information-theoretic perspective, it suffices to minimize the conditional mutual information $I(Z; A \mid Y)$ to ensure EO due to the celebrated data-processing inequality. We identify a connection between contrastive learning and learning fair representations when the representations enjoy certain benign structures. Next, we formally state the assumptions to characterize such a structure.
\begin{assumption}
\label{assump:epsilon}
Let $Z$ and $Z'$ be the corresponding features from $X$ and $X'$, respectively. We assume that there exists a small positive constant $\epsilon > 0$, such that $H(Z\mid Z', Y)\leq\epsilon$.
\end{assumption}
At a high level, Assumption~\ref{assump:epsilon} says that the learned features from the contrastive learning procedure are well conditionally aligned~\citep{wang2020understanding}. 
Specifically, given the label of a feature and its corresponding augmented feature, it is relatively easy to infer the corresponding positive pair used in the contrastive learning procedure. 
Note that the conditional entropy could be understood as the minimum inference error from this perspective~\citep{farnia2016minimax}. Under Assumption~\ref{assump:epsilon}, we provide the following lemma to characterize the relationship between $Z$, $Z'$, $A$, and $Y$ in terms of (conditional) mutual information. 

\begin{lemma}
Under Assumption~\ref{assump:epsilon}, given a set of data augmentations $\mathcal{T}$, let $X'$ be the augmented input data where $X'=t(X),~ t\sim\mathcal{T}$. Assuming the following Markov chains $X\overset{f}{\to} Z \overset{g}{\to} \hat{Y}$ and $X' \overset{f}{\to} Z' \overset{g}{\to} \hat{Y}'$ hold,
we have
\begin{align*}
&\, I(Z'; Z \mid Y) - I(Z'; Z \mid A, Y) - \epsilon \\
\leq \ &\,  I(Z; A\mid Y) \\
\leq \ &\,  I(Z'; Z \mid Y) - I(Z'; Z \mid A, Y) + \epsilon.
\end{align*}
\label{prop:cmi-aug}
\end{lemma}
Lemma~\ref{prop:cmi-aug} indicates that we can minimize $I(Z; A \mid Y)$ via (1) minimizing $I(Z'; Z \mid Y)$ and (2) maximizing $I(Z'; Z \mid A, Y)$. In what follows, we will present an upper (lower) bound to minimize $I(Z'; Z \mid Y)$ (maximize $I(Z'; Z \mid A, Y)$) and connect the bounds with contrastive learning objectives. We first provide an upper bound of $I(Z'; Z \mid Y)$. 

\begin{proposition}
Given the assumptions in Lemma~\ref{prop:cmi-aug}, we have
\begin{align*}
&\quad\, I(Z';  Z \mid Y) \\
& \leq - \mathbb{E}_{p(y)} \big[ \mathbb{E}_{p(z'\mid y)}[\mathbb{E}_{p(z\mid y)}[\log p(z'\mid z, y)]] \big].
\end{align*}
\label{thm:cmi-sup-constrative-upper}
\end{proposition}

In order to better interpret the right side in  Proposition~\ref{thm:cmi-sup-constrative-upper}, we define a similarity function $s(z', z; y)$ between $z'$ and $z$ for each $y$ and assume $s(z', z; y) \propto p(z'\mid z, y)$ (\ie, the more similar $z'$ and $z$ are in the latent space given task label $y$, the more likely $z'$ is generated by $z$ via data augmentation)\footnote{In the remainder of the paper, we let $s(\cdot,\cdot) = s(\cdot,\cdot; y),~\forall~Y=y$ for the ease of practical implementations.}. With this assumption, the upper bound provided in Proposition~\ref{thm:cmi-sup-constrative-upper} implies that $I(Z'; Z \mid Y)$ can be minimized by encouraging similarity between any latent representations given the same task label, which is consistent with the goal of supervised contrastive loss~\citep{khosla2020supervised}. Formally, given a batch of augmented examples ${(x_i, y_i, a_i)_{i=1}^{2N}}$ with size $2N$, where the last half examples of the batch are the augmented views of the first half and they share the same task labels (as well as the same sensitive attributes), \ie, $x_{i+N} = t(x_i)$ for $i\in[N]$ and $t\sim \mathcal{T}$. 
Let $N_{y_i}$ be the total number of examples in the batch that have the same task label as $y_i$, then supervised contrastive loss takes the following form:
\begin{equation}
    L_{\text{sup}} = -\sum_{i=1}^{2N} \frac{1}{N_{y_i}-1} \sum_{j=1}^{2N} \mathbf{1}_{i\neq j, y_i = y_j} \log(\ell_{ij}),
\end{equation}
and $\ell_{ij}$ is defined as
\begin{align*}
\nonumber
\ell_{ij} & = \frac{\exp\big(f(x_i) \cdot f(x_j) / \tau \big)}{\sum_{k=1}^{2N} \mathbf{1}_{i\neq k} \exp\big(f(x_i) \cdot  f(x_k) /\tau \big) },
\end{align*}
where $\tau$ is the temperature parameter, $\mathbf{1}_{i\neq k} = \mathbf{1}\{i\neq k\}$ and $\mathbf{1}\{\cdot\}$ is the indicator function, the similarity function is $s\big(f(x_i), f(x_j)\big) = \exp\big(f(x_i) \cdot  f(x_j) /\tau \big)$, and the $\cdot$ symbol denotes the inner (dot) product.
Supervised contrastive loss $L_{\text{sup}}$ aims to encourage similarity between different examples with the same task label and discourage the ones having different labels. Thus, we minimize $L_{\text{sup}}$ to approximately minimize $I(Z'; Z \mid Y)$.
Next, we provide a lower bound of $I(Z'; Z \mid A, Y)$ for the maximization of $I(Z'; Z \mid A, Y)$. 
\begin{proposition}
Given the assumptions in Lemma~\ref{prop:cmi-aug}, define conditional supervised InfoNCE as \text{CS-InfoNCE}, \ie,

{\footnotesize $$\underbrace{\sup_{s} \mathbb{E}_{p(a, y)} \bigg[ \mathbb{E}_{p(z'_i, z_i \mid a, y)^{\otimes N}} \big[ \log \frac{\exp(s(z'_i, z_i))}{\frac{1}{N} \sum_{j=1}^N \exp(s(z'_i, z_j)) } \big] \bigg]}_{\large \text{CS-InfoNCE}},$$}

\noindent
where $p(\cdot)^{\otimes N}$ denotes the probability distributions of $N$ independent examples 
and $s(\cdot, \cdot)$ is any similarity function that measure the similarity of $z'_i$ and $z_i$. Then, we have 
\begin{align}
\nonumber
\text{CS-InfoNCE} \leq I(Z'; Z \mid A, Y).
\end{align}
\label{thm:cmi-sup-constrative-lower}
\end{proposition}

Proposition~\ref{thm:cmi-sup-constrative-lower} indicates the maximization of $\text{CS-InfoNCE}$ leads to the maximization of $I(Z'; Z \mid A, Y)$. Given the examples that share the same task label and sensitive attribute, $\text{CS-InfoNCE}$ encourages the similarity between different views of the same examples while discouraging others. Note that all positive and negative examples w.r.t. the anchoring example share the same task label and sensitive attribute. Given the same batch of examples ${(x_i, y_i, a_i)_{i=1}^{2N}}$ with size $2N$, we can formulate the contrastive objective as
\begin{align}
    L_{\text{CS-InfoNCE}} & = - \sum_{i=1}^{2N} \frac{1}{N_{a_i, y_i}-1}\log{(\ell_i)},
\end{align}
and $\ell_i$ is defined as
\begin{align*}
        \nonumber
    \ell_i & \!=\!
    \frac{\exp\left(f(x_i) \cdot f(x_i')/\tau\right)}{\sum_{k=1}^{2N} \mathbf{1}_{i\neq k,a_i=a_k, y_i=y_k} \exp\left(f(x_i) \cdot f(x_k)/\tau \right)},
\end{align*}
where $N_{a_i, y_i}$ is the total number of examples in the batch that have the same task label and sensitive attribute as $y_i$ and $a_i$, and $x_i$ and $x_i'$ are the different views of the same example.

\vspace{0.05in}
\noindent
\textbf{Interpretation of} $L_{\text{sup}}$ \textbf{and} $L_{\text{CS-InfoNCE}}$ \textbf{in learning fair representations.}
In learning fair representations, the role of $L_{\text{sup}}$ is to learn aligned and uniform representations~\citep{wang2020understanding} for each task label, while the role of $L_{\text{CS-InfoNCE}}$ is to encourage the dissimilarity of different examples that share the same task labels and sensitive attributes. In an ideal case where $L_{\text{sup}}=0$, each data point that shares the same task label in the latent space collapse to a single point, and the perfect representations are learned. In this case, $L_{\text{CS-InfoNCE}}=0$ as well. In practice, the overall combined effect of the $L_{\text{sup}}$ and $L_{\text{CS-InfoNCE}}$ will encourage the similarity of examples having the same task label but belonging to different groups. 
Thus, our theory could also explain why other slightly different proposed contrastive objectives in a concurrent work~\citep{park2022fair} could mitigate equalized odds.
In Appendix~\ref{app-subsec:tsne-visualization}, we provide T-SNE visualization~\citep{van2008visualizing} of the text embeddings using different training objectives to help better understand our methods.

\subsection{Practical Implementations} \label{subsec:methodology:practical-implementation}

The existing contrastive representation learning approaches fall into two categories: two-stage methods~\citep{khosla2020supervised, pmlr-v119-chen20j} and one-stage methods~\citep{gunel2021supervised, cui2021parametric}. Two-stage methods first pretrain the encoder in the first stage using the contrastive objective, then fix the encoder, and fine-tune the classifier using cross-entropy (CE) loss in the second stage. One-stage methods train both encoder and classifier using CE loss and contrastive loss end-to-end. Following the previous settings, we also implement our methods in these two ways. For the two-stage CL method, we first pretrain the text encoder using the following loss function in the first stage:
\begin{equation}
   L_{\text{sup}} + \lambda\cdot L_{\text{CS-InfoNCE}},
   \label{eq:two-stage-cl}
\end{equation}
then we fix the pretrained encoder, and fine-tune classifier using CE loss. Note that $\lambda \geq 0$ controls the intensity of $L_{\text{CS-InfoNCE}}$. 
For the one-stage CL method, similar to~\citet{gunel2021supervised}, we formulate the loss function as:
\begin{equation}
   (1-\gamma)\cdot L_{\text{CE}} + \gamma\cdot L_{\text{sup}} + \lambda\cdot L_{\text{CS-InfoNCE}},
   \label{eq:one-stage-cl}
\end{equation}
where $\gamma\in [0, 1]$ controls the relative weight of $L_{\text{sup}}$ compared to $L_{\text{CE}}$.
The major advantage of our approach is that it can be directly substituted into existing NLP pipelines that use the ``pretrain-and-finetune'' paradigm popularized by large language models such as BERT. 
NLP practitioners can swap the fair CL finetuner into these pipelines to boost model fairness at low cost, with robust behavior against hyperparameter choices (see Sec.~\ref{subsec:results-and-analysis}). Whereas large language models made it simple to build models with high performance, fair CL makes it simple to build models with high performance and fairness.

\section{Experiments}

In this section, we conduct experiments to investigate the following research questions:

\begin{itemize}
    \item[\textbf{RQ 1.}] How can we control the trade-offs between model classification performance and fairness via conditional supervised contrastive learning?
    \item[\textbf{RQ 2.}] How do conditional supervised contrastive learning methods perform in terms of trade-offs between model performance and fairness compared to other in-processing bias mitigation methods in text classification?
    \item[\textbf{RQ 3.}] Is conditional supervised contrastive learning sensitive to hyperparameter changes?
\end{itemize}

\subsection{Experimental Setup}

\paragraph{Datasets.} We perform experiments using the following two datasets (see Appendix~\ref{app:exp-details} for more details of the datasets and the data prepossessing pipelines): 
\begin{itemize}
    \item \texttt{Jigsaw-toxicity}\footnote{The dataset is publicly available at: \url{https://www.kaggle.com/competitions/jigsaw-unintended-bias-in-toxicity-classification}.} is a dataset for online comment toxicity classification. The main task of the dataset is to determine if the online comment is toxic, and we use ``race and ethnicity'' as the sensitive attribute (\eg, whether ``black'' identity is mentioned in the comment text or not). 
    \item \texttt{Biasbios}~\citep{de2019bias} is a dataset for occupation classification. The main task of the dataset is to determine the people's occupations given their biographies. The sensitive attribute is binary gender (\ie, male and female). 
\end{itemize}

\paragraph{Evaluation Metrics.}
We evaluate our model based on model classification performance and EO fairness. We use the F1 score for model performance and True Positive Equality Difference + False Positive Equality Difference~\citep{dixon2018measuring} for EO fairness:
\begin{align*}
    \Delta_{\text{TPR}} & = \sum_a | \text{TPR}_a  - \text{TPR}_{\text{overall}} |, \\
    \Delta_{\text{FPR}} & = \sum_a | \text{FPR}_a  - \text{FPR}_{\text{overall}} |, 
\end{align*}
where $\Delta_{\text{TPR}}$ ($\Delta_{\text{FPR}}$) is the true positive rate (false negative rate) for sensitive attribute $a$ and $\text{TPR}_{\text{overall}}$ ($\text{FPR}_{\text{overall}}$) is the overall true positive rate (false negative rate). Following~\citet{pruksachatkun-etal-2021-robustness}, we define equalized odds gap $\Delta_{\text{EO}} = \Delta_{\text{TPR}} + \Delta_{\text{FPR}}$, since equalized odds aligns with $\Delta_{\text{TPR}}+\Delta_{\text{FPR}}$, and when it is satisfied, $\Delta_{\text{TPR}}=\Delta_{\text{FPR}}=0$~\citep{borkan2019nuanced}. Note that when $|\yspace| > 2$, $\Delta_{\text{EO}}$ will be summed over each value in $\yspace$ since TPR and FPR are defined over each class (\ie, $\Delta_{\text{EO}} = \sum_y \Delta_{\text{EO}}^y$).

\paragraph{Implementations and Baselines.} 
In our experiments, we use BERT~\citep{devlin-etal-2019-bert} (\texttt{bert-base-uncased} as the text encoder followed by a two-layer MLP as the classifier)\footnote{We use the huggingface transformer implementation: \url{https://github.com/huggingface/transformers}.}. 
As suggested by previous works~\citep{khosla2020supervised, gao2021simcse}, the performance of contrastive learning is closely related to the choice of the following hyperparameters: (1) temperature, (2) (pre-training) batch size, and (3) data augmentation strategy. 
Thus, we conduct a grid-based hyperparameter search for temperature $\tau\in \{0.1, 0.5, 1.0, 2.0\}$,  (pre-training) batch size $\mathrm{bsz}\in \{32, 64, 128, 256\}$, and data augmentation strategy $t \in$ \{\text{EDA}, \text{back translation}, \text{CLM insert}, \text{CLM substitute}\} (see Appendix~\ref{app:exp-details} for the detailed description of different augmentation strategies) for both two-stage CL and one-stage CL. We also conduct grid search of $\gamma\in\{0.1, 0.3, 0.7, 0.9\}$ in Eq.~\eqref{eq:one-stage-cl} for one-stage CL.
In Appendix~\ref{app:exp-details}, we provide the remaining hyperparameter details (\eg, learning rate, training epochs, optimizer). Since it is not feasible to train large language models with large batch sizes via contrastive objectives given limited GPU memory, we use the gradient cache technique~\citep{gao-etal-2021-scaling} to adapt our implementations to limited GPU memory settings. 

We compare our methods with the following baselines, which have been empirically demonstrated effective for bias mitigation in text classification: 
\begin{enumerate}
    \item[(1)] Adversarial training~\citep{elazar-goldberg-2018-adversarial}: Following the encoder + classifier setting, adversarial training leverages a discriminator to learn latent representations oblivious to the sensitive attribute.
    Note that the original adversarial training method is tailored for demographic parity and it is well known that demographic parity and equal odds are incompatible given different base rates~\citep{Kleinberg2017InherentTI, ball2021differential}. To this end, we use the conditional learning techniques~\citep{pmlr-v80-madras18a, zhao2019conditional} to adapt adversarial training for equalized odds.
    \item[(2)] Adversarial training with diverse adversaries  (diverse adversaries)~\citep{han-etal-2021-diverse}: Adversarial training with diverse adversaries improves adversarial training by using an ensemble of discriminators and encourages the discriminator to learn orthogonal representations. Similar to adversarial training, we also apply the conditional learning techniques for learning the adversarial discriminators.
    \item[(3)] Iterative null-space projection (INLP)~\citep{ravfogel-etal-2020-null}: Given a pretrained text encoder (we use CE loss to pretrain the text encoder and drop the prediction head using the validation set), INLP learns a linear guarding layer on top of the pretrained text encoder to filter the sensitive information and fine-tune the classifier given the pretrained text encoder and INLP. INLP learns the linear guarding layer by projecting the parameter matrices of linear classifiers (\eg, SVM) to their null spaces iteratively. The training data of linear classifiers are the latent representations of input texts and sensitive attributes. In order to tailor INLP for equaled odds, \citet{ravfogel-etal-2020-null} learns the linear classifier given the data from the same class each round.
\end{enumerate}
We also use training using CE loss as a baseline. 
Except for INLP, all methods we test in our experiments train the text encoder (\eg, BERT) directly, while INLP is a post-hoc debiasing method given a text encoder. 
In a sense, INLP is orthogonal to other methods since it tries to remove group-specific information after we learn the representations, while other methods learn the fair representations directly.
We run each experiment with five different seeds and report the mean and standard deviation values for each evaluation metric.

\subsection{Results and Analysis} \label{subsec:results-and-analysis}
\paragraph{RQ 1.} In order to control the trade-offs between model classification performance and EO fairness, we vary the values of $\lambda$ in Eq.~\eqref{eq:two-stage-cl} and Eq.~\eqref{eq:one-stage-cl}. Figure~\ref{fig:rq1} shows the classification performance and EO fairness of one-stage and two-stage CL when $\lambda$ changes. 
Overall, as $\lambda$ increases, the equalized odds gaps shrink at the cost of model classification performance.
Compared to one-stage CL, two-stage CL achieves more flexible trade-offs in general. Given the same range of $\lambda$, the change of equalized odds gaps in two-stage CL is more significant than in one-stage CL. At the same time, the corresponding model classification performances are comparable or remain better.

\begin{figure*}
     \centering
     \begin{subfigure}[b]{0.48\textwidth}
         \centering
         \includegraphics[width=\textwidth]{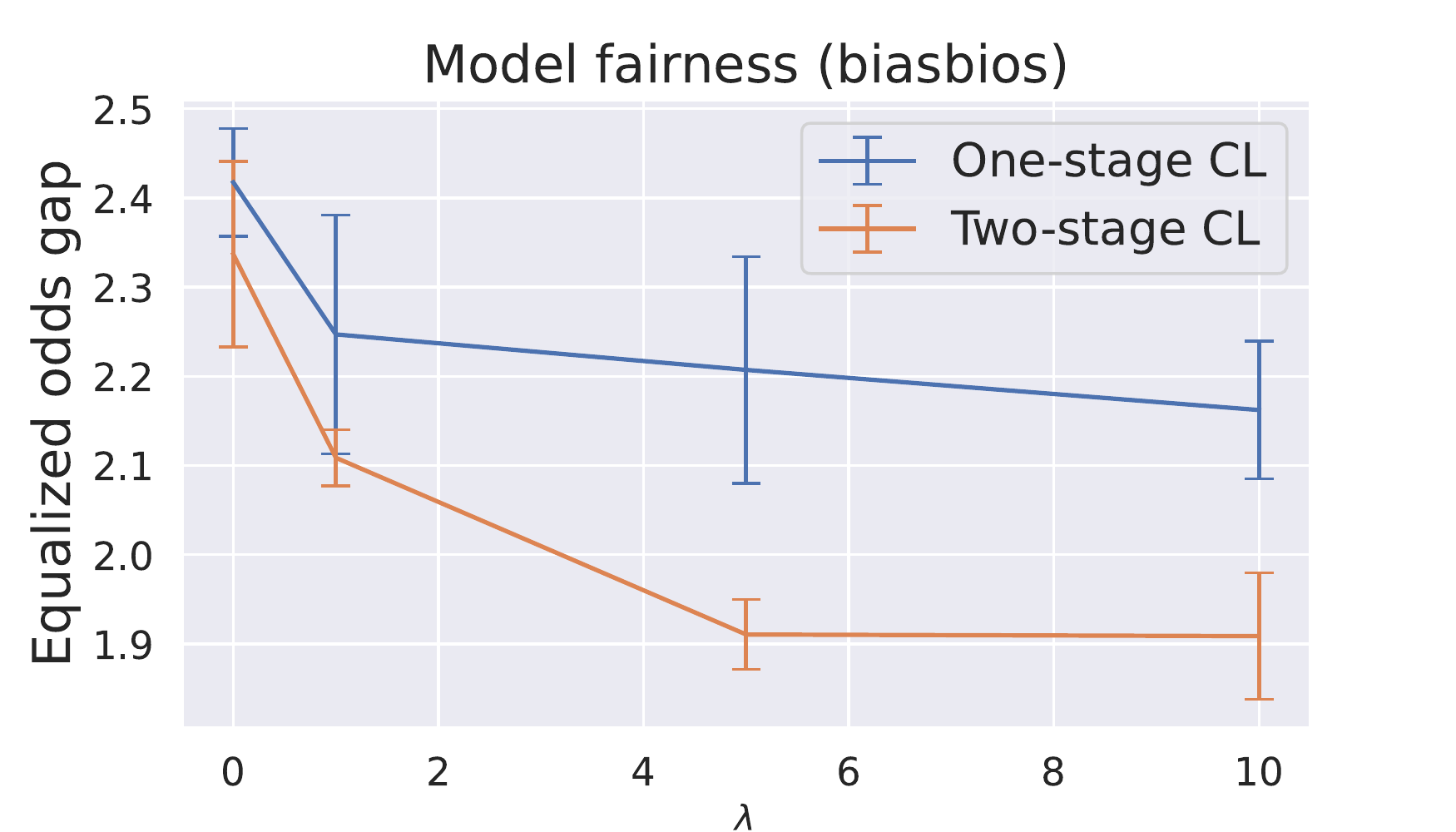}
     \end{subfigure}
     \hfill
     \begin{subfigure}[b]{0.48\textwidth}
         \centering
         \includegraphics[width=\textwidth]{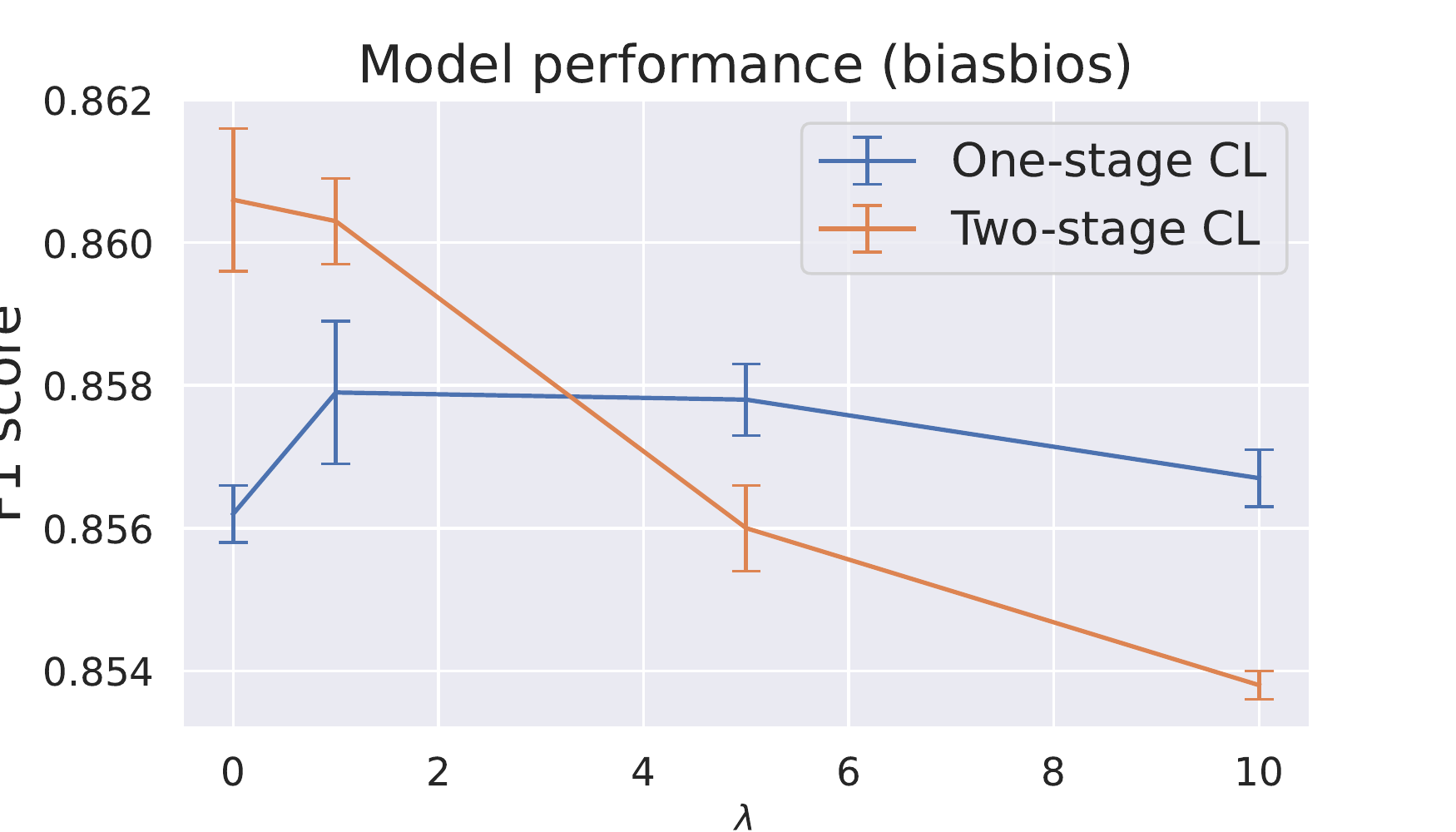}
     \end{subfigure}
     \begin{subfigure}[b]{0.48\textwidth}
         \centering
         \includegraphics[width=\textwidth]{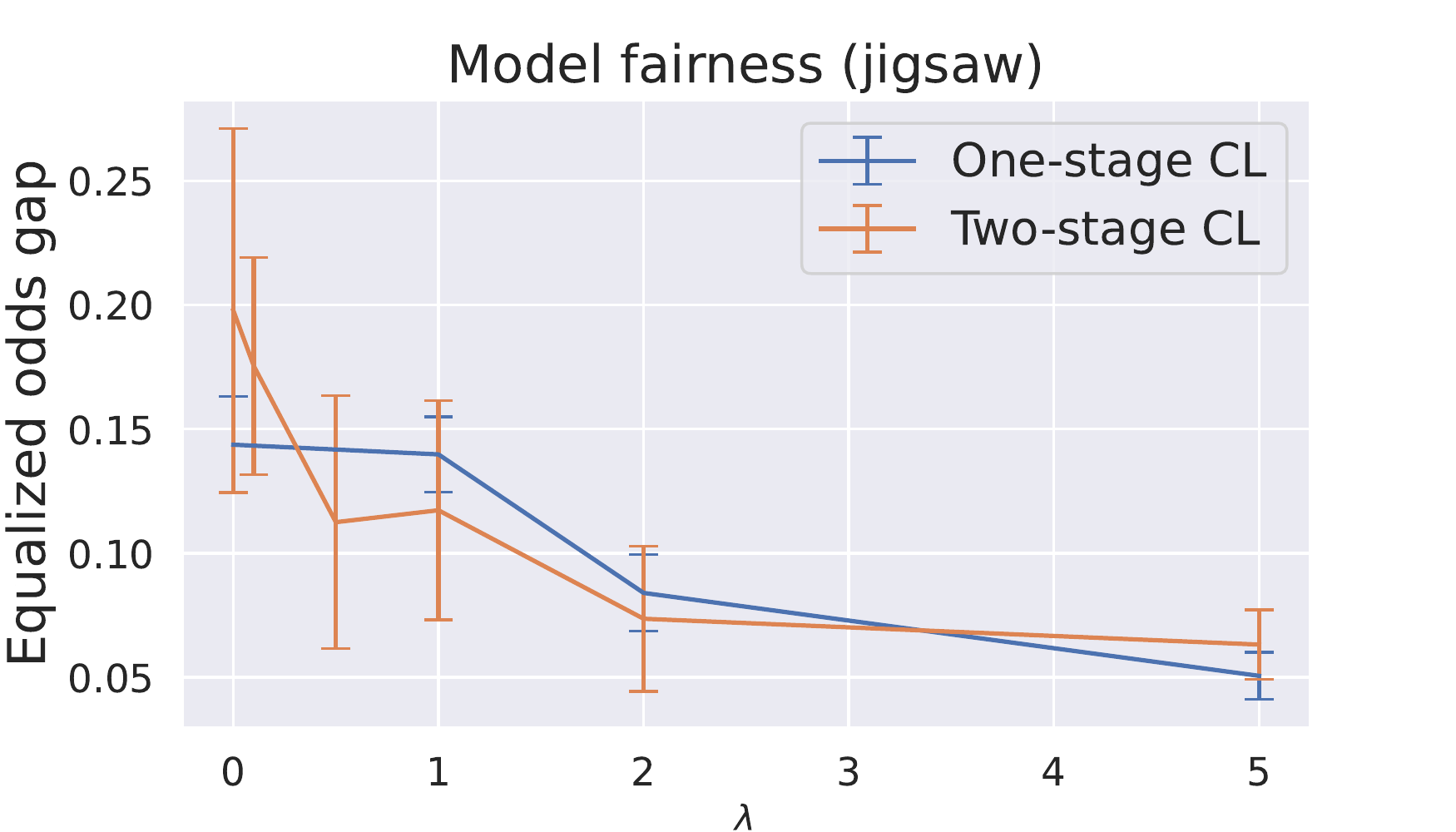}
     \end{subfigure}
     \hfill
     \begin{subfigure}[b]{0.48\textwidth}
         \centering
         \includegraphics[width=\textwidth]{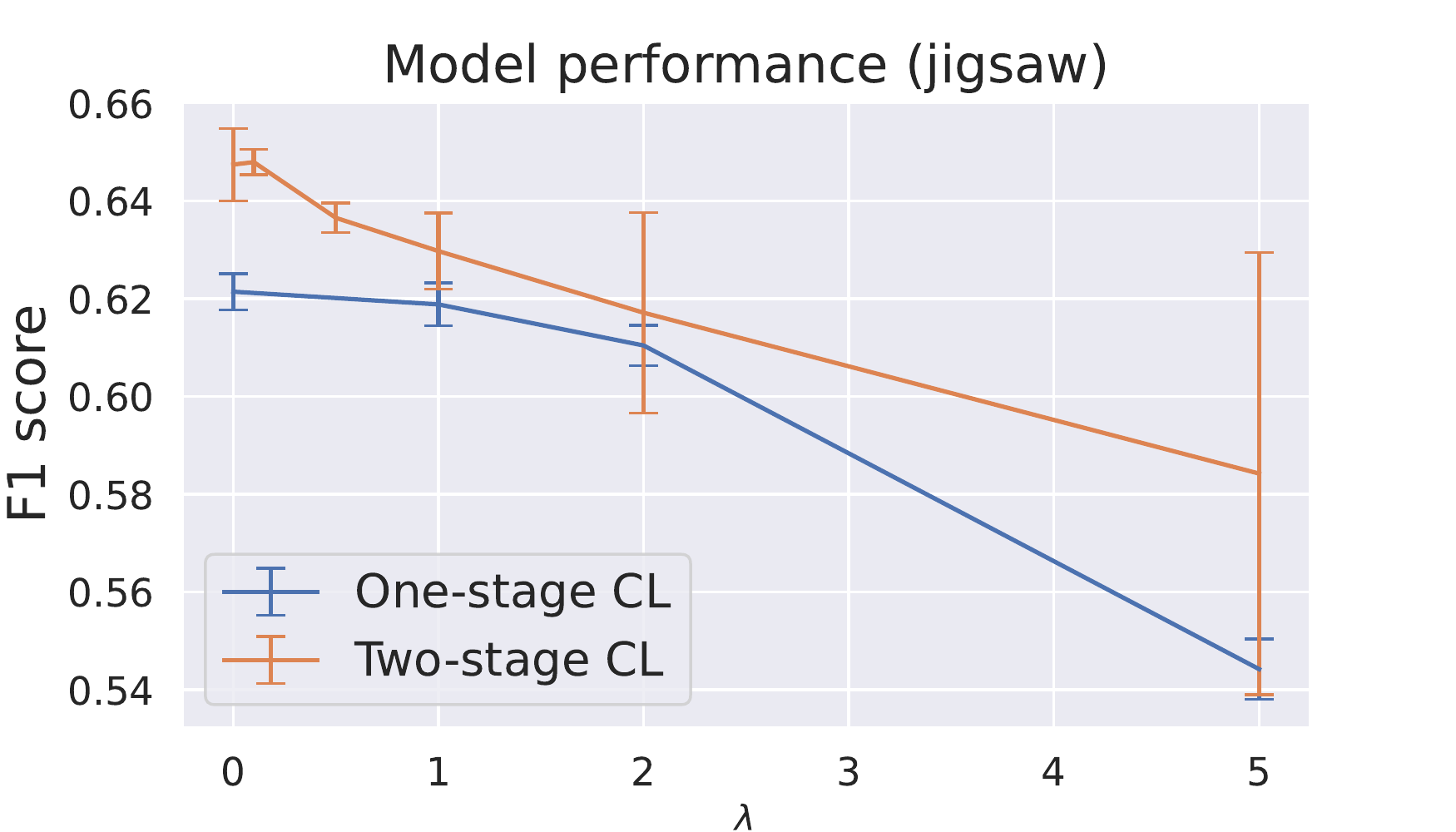}
     \end{subfigure}
     \caption{Classification performance and EO fairness of one-stage and two-stage CL when $\lambda$ changes. The equalized odds gaps shrink at the cost of model classification performance as $\lambda$ increases.}
     \label{fig:rq1}
\end{figure*}

\paragraph{RQ 2.} 
We study the trade-offs between model performance and EO fairness of our proposed methods compared to the baselines. Figure~\ref{fig:results:rq2-fair-perf} displays the performance and fairness of these methods under different hyperparameter settings for the \texttt{jigsaw} and \texttt{biasbios} datasets 
(trade-off parameters for all methods are described in more detail in Appendix~\ref{app:exp-details}).

\begin{figure*}[!t]
     \centering
     \begin{subfigure}[b]{0.48\textwidth}
         \centering
         \includegraphics[width=\textwidth]{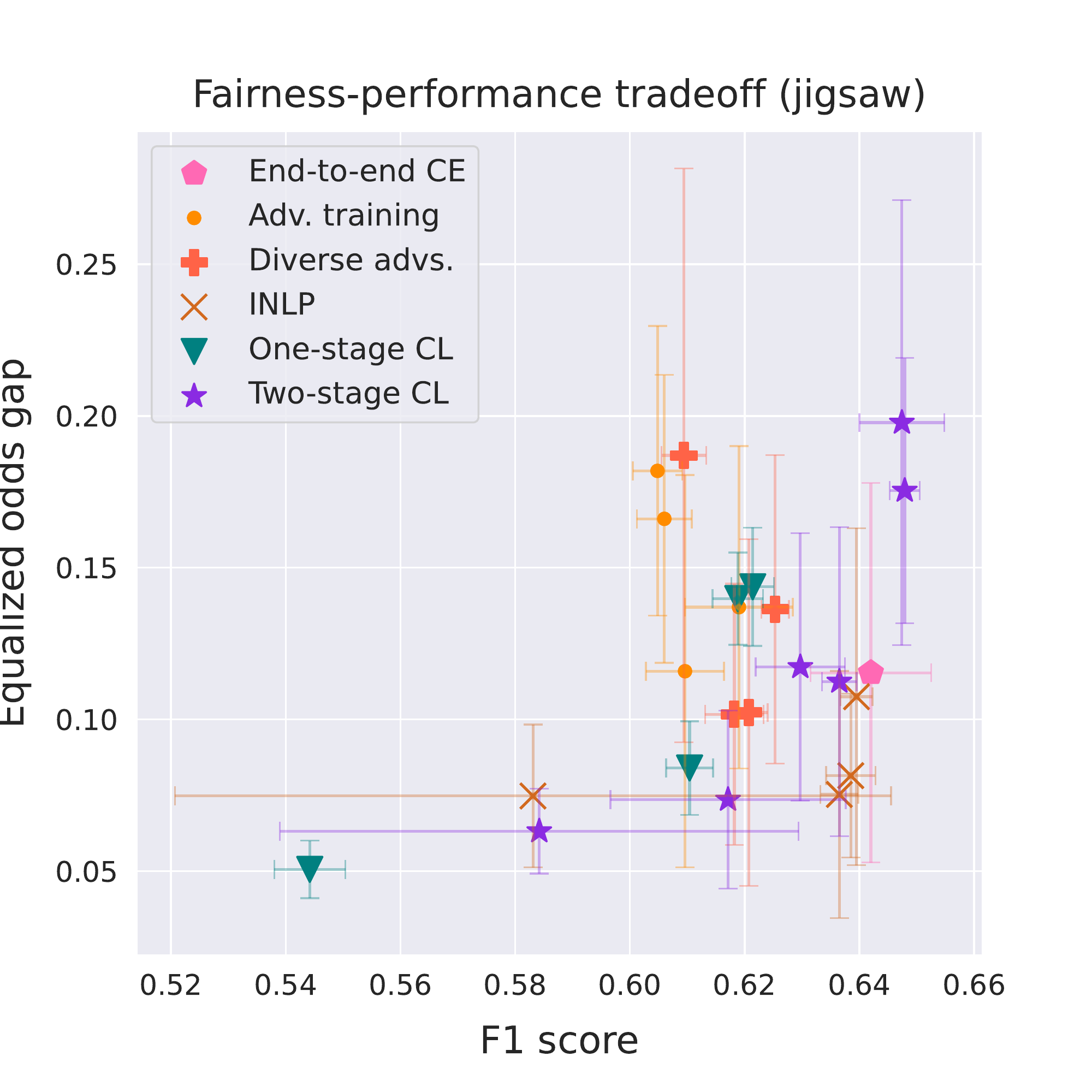}
     \end{subfigure}
     \hfill
     \begin{subfigure}[b]{0.48\textwidth}
         \centering
         \includegraphics[width=\textwidth]{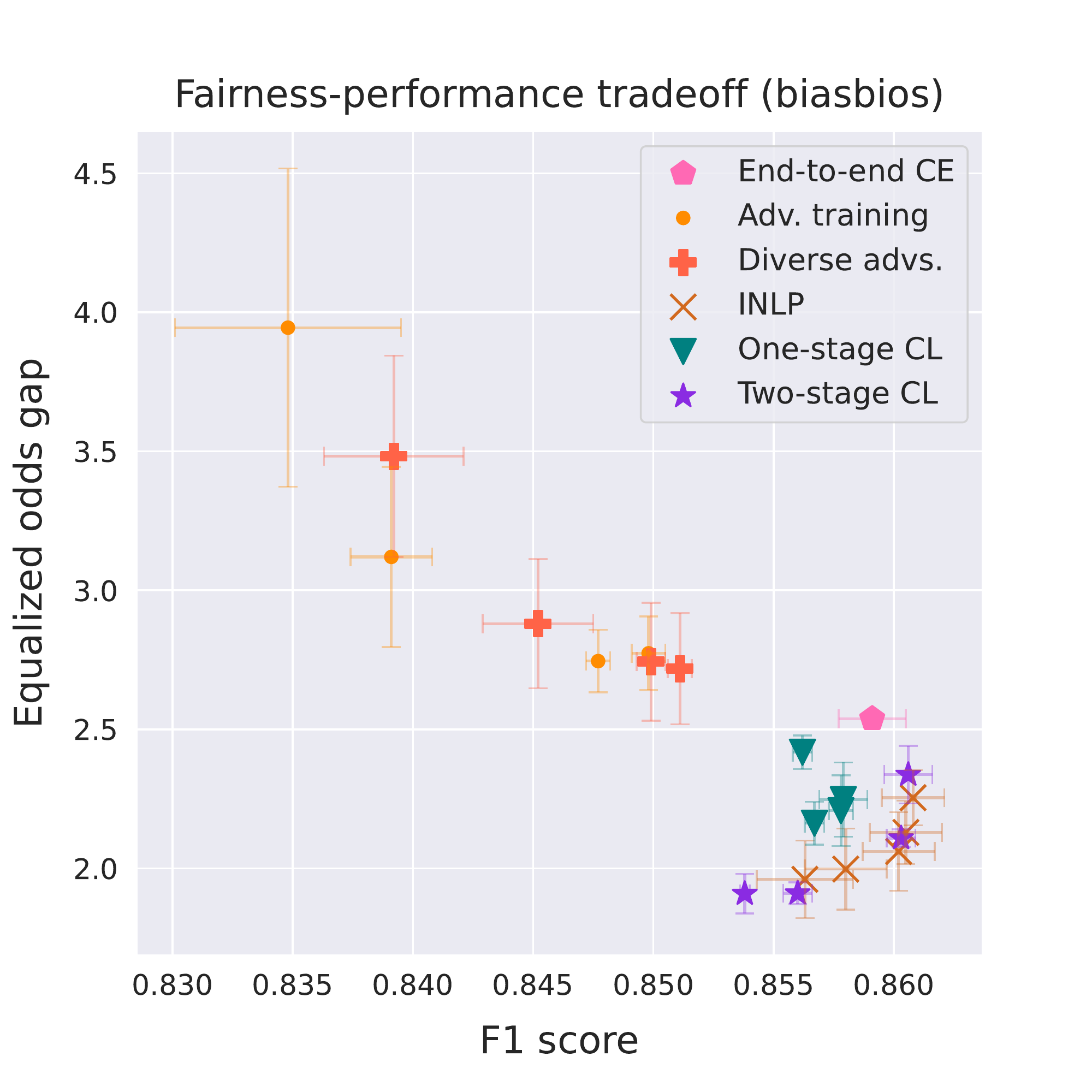}
     \end{subfigure}
     \caption{Classification performance and EO fairness and of our proposed methods compared against the baselines. Two-stage CL and INLP achieves the best performance and fairness trade-offs in general, and two-stage CL typically achieves more consistent results with lower variance.}
     \label{fig:results:rq2-fair-perf}
\end{figure*}

Among all methods, we find that two-stage CL and INLP achieve the best performance and fairness trade-offs. 
In the \texttt{biasbios} dataset, two-stage CL and INLP achieve similar performance and fairness trade-offs, and two-stage CL achieves more consistent results (\ie, lower variance). 
In the \texttt{jigsaw} dataset, two-stage CL achieves more flexible performance and fairness trade-offs as it reaches the highest model performance.  
Besides, when F1 scores are around 0.58, two-stage CL also achieves more consistent results and a lower EO gap.
Meanwhile, when F1 scores are between 0.62$\sim$0.64, INLP performs better.
We note that the effectiveness of INLP highly depends on the pretrained encoder for INLP (see Appendix~\ref{app-subsec:inlp} for the effects of different pre-training strategies for the text encoder in INLP), and a slight change in the text encoder could lead to a significant difference in the results, while CL-based methods target training the text encoder directly to ensure EO fairness and we demonstrate they are stable under hyperparameter changes (see RQ 3 below).

In comparison, the adversarial-training-based methods are relatively more unstable and consistently perform worse than CL-based methods and INLP, especially in the \texttt{biasbios} dataset.
Furthermore, both adversarial-training-based methods and INLP introduce additional model components (\eg, adversarial networks in adversarial-training-based methods and linear guarding layer in INLP) during training or inference, which complicates the actual implementation of the whole pipeline. In contrast, CL-based methods are well-suited to pre-training and fine-tuning paradigms in NLP applications.


\begin{figure*}[!ht]
    \begin{subfigure}[b]{\textwidth}
    \centering
    \includegraphics[width=\textwidth]{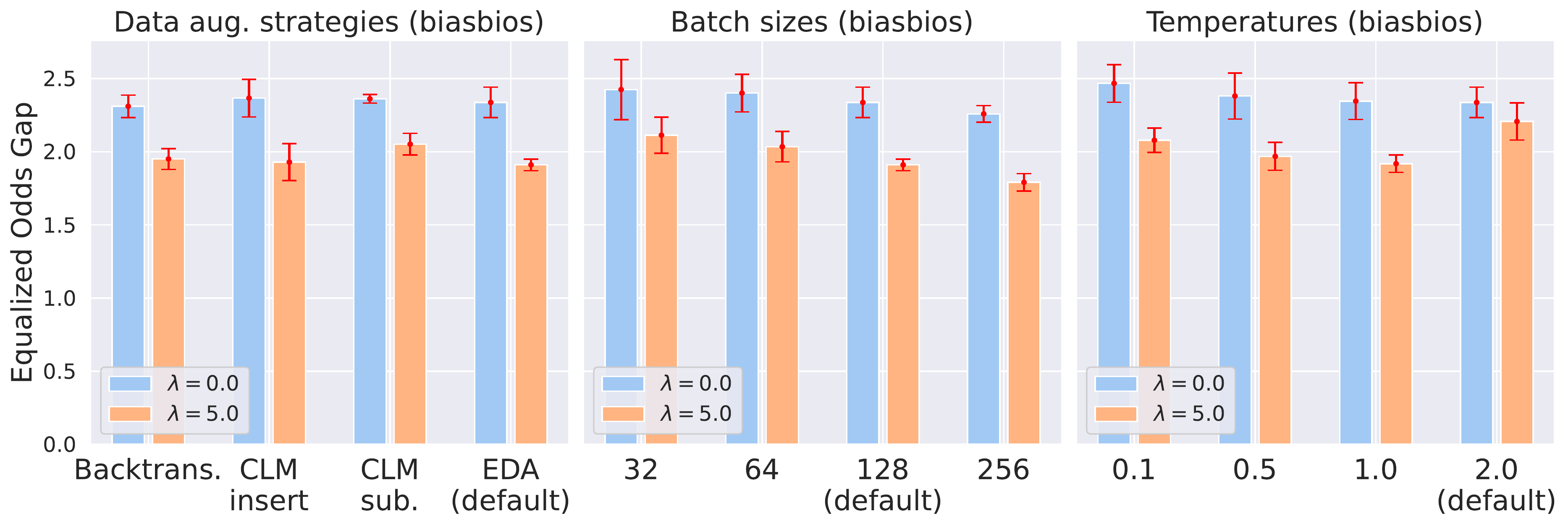}
    \end{subfigure} \begin{subfigure}[b]{\textwidth}
    \centering
    \includegraphics[width=\textwidth]{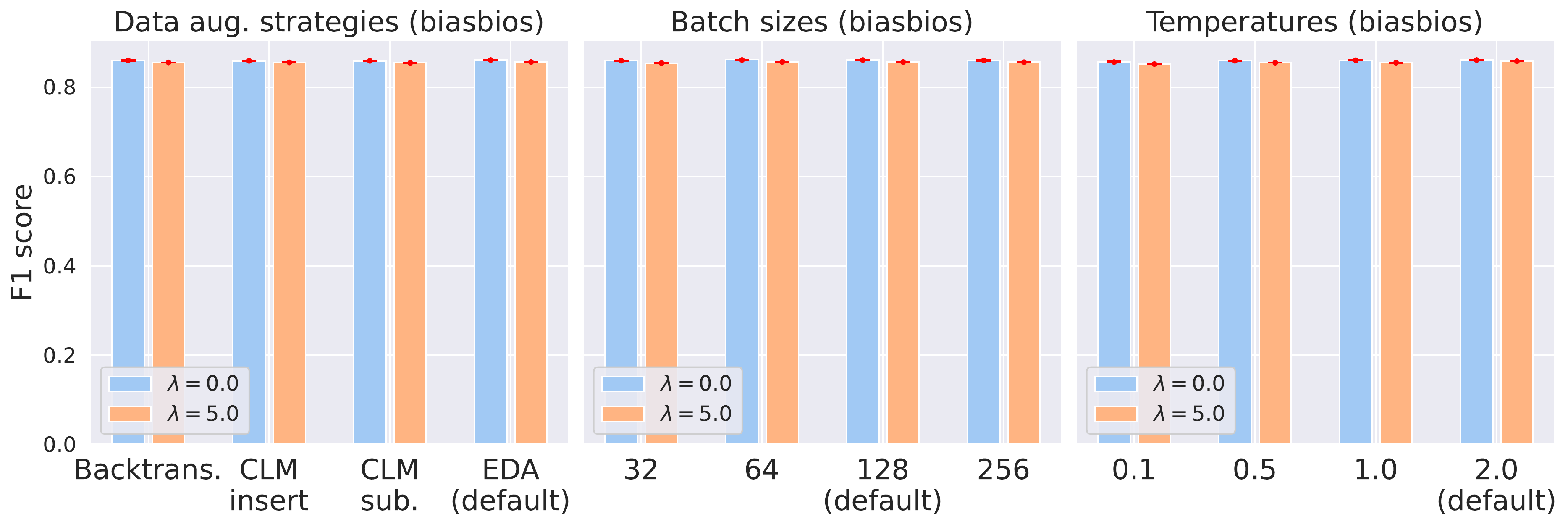}
    \end{subfigure}
    \caption{Sensitivity analysis of two-stage CL to key hyperparameter changes in (\texttt{biasbios}). ``Default'' in the X-axis indicates the default hyperparamter settings used in RQ 1 and RQ 2. }
    \label{fig:rq3-biasbios}
\end{figure*}

\paragraph{RQ 3.} 

We have shown that two-stage CL performs better than one-stage CL in RQ 1 and RQ 2. Thus, we choose two-stage CL to see if it is sensitive to key hyperparameter changes. As mentioned above, the performance of contrastive learning is closely related to temperature, (pre-training) batch size, and data augmentation strategy. Thus, we study whether the performance of two-stage CL is sensitive to these hyperparameters.

Figure~\ref{fig:rq3-biasbios} shows model performance and EO fairness of two-stage CL under different hyperparameter settings when $\lambda\in\{0.0, 5.0\}$ in the \texttt{biasbios} dataset (Figure~\ref{fig:rq3-jigsaw} for the \texttt{jigsaw} in Appendix~\ref{app-subsec:cl}). We see that two-stage CL are stable under a wide range of parameter settings: The equalized odds gaps are consistently decreasing when $\lambda=5.0$ and the F1 scores are relatively high.

\section{Related Work}

Unintended social biases in NLP models have been identified in word/sentence embedding~\citep{bolukbasi2016man, may2019measuring,zhao-etal-2019-gender} and applications such as coreference resolution~\citep{zhao-etal-2018-gender,rudinger-etal-2018-gender,cao-daume-iii-2020-toward}, language modeling~\citep{bordia2019identifying}, machine translation~\citep{stanovsky-etal-2019-evaluating}, and text classification~\cite{ball2021differential,baldini-etal-2022-fairness}.

In the literature, there are some recent works that aim to learn fair representations via contrastive learning~\citep{cheng2021fairfil, shen2021contrastive, tsai2021conditional, tsai2022conditional}.
Among these works,
\citet{cheng2021fairfil} propose contrastive objectives to learn debiased sentence embeddings that minimize the correlation between embedded sentences and biased words. In classification tasks,
\citet{tsai2021conditional, tsai2022conditional} proposed contrastive objectives to remove sensitive information; \citet{shen2021contrastive} proposed a similar contrastive objective to achieve a similar goal. According to~\citet{tsai2021conditional}, all those proposed contrastive objectives target demographic parity in principle. 

Our theoretical results involve key notions (\eg, entropy and mutual information) in information theory~\citep{cheng-etal-2020-improving, colombo-etal-2021-novel}.
Information-theoretic-based methods have been used for representation learning for NLP applications. For example, \citet{colombo-etal-2021-novel} proposed a variational upper bound of mutual information to learn disentangled textual representations for fair classification and style transfer. 

Compared to the previous work, our work uses equalized odds as the fairness criterion. To the best of our knowledge, our work is the first to connect the problem of learning fair representations with contrastive learning to ensure the EO constraint and explore its effectiveness for bias mitigation in text classification in large language models (\eg, BERT).

\section{Conclusion}

In this paper, we theoretically and empirically study how to leverage contrastive learning for fair text classification. Inspired by our theoretical results, we propose conditional supervised contrastive objectives to learn aligned and uniform representations while mixing the representation of different examples that share the same sensitive attribute for every task label.
We conduct experiments to demonstrate the effectiveness of our algorithms in learning fair representations for text classification and show that our methods are stable in different hyperparameter settings.
In the future, we plan to extend our algorithms to the settings of intersectional bias~\citep{pmlr-v80-kearns18a, yang2020fairness}.

\section*{Limitations}

Like most prior work~\citep{ravfogel-etal-2020-null, tsai2022conditional}, we conduct experiments on the binary sensitive attribute. Especially, we acknowledge that due to the limitation of the dataset, our analysis of gender bias only considers binary gender, which is not ideal~\cite{dev-etal-2021-harms}.
One interesting future direction is to extend our method to ensure fairness for intersectional groups~\citep{pmlr-v80-kearns18a, yang2020fairness}. 
In principle, our theory also holds in intersectional bias. However, the disproportionate distributions of the intersectional sensitive attributes might pose challenges in sampling negative examples in $L_{\text{CS-InfoNCE}}$. One possible solution is to use a memory bank to sample negative examples~\citep{Wu_2018_CVPR}.
We leave this analysis as future work as it is an important question that warrants an independent study. 


\section*{Ethics Statement}

This work aims for bias mitigation for text classification. Like other bias mitigation methods, it could help increase people's trust in NLP models.
For example, our methods could help reduce unintentional censoring (false positive cases) and debates or discomforts (false negative cases) in online comment toxicity classification (see Section~\ref{sec:background} for more details).
Our study targets equalized odds and do not capture all notions of bias (\eg, individual fairness) in text classification. These issues are universal to bias mitigation techniques and not particular to our use case.

\section*{Acknowledgements}
We thank the anonymous reviewers for their insightful comments. 
Jianfeng Chi, William Shand, and Yuan Tian acknowledge support from NSF \#1829004, \#1920462, \#2002985, a Facebook Faculty Fellowship, and a Google Research Scholar Award.
Yaodong Yu acknowledges support from the joint Simons Foundation-NSF DMS grant \#2031899. Han Zhao would like to thank the support from a Facebook research award. Kai-Wei Chang acknowledge support from NSF \#1927554 and a Sloan Research Fellow.

\bibliography{anthology,custom}
\bibliographystyle{acl_natbib}

\newpage

\appendix

\section{Omitted Proofs}
\label{app:proof}

\subsection{Proof of Lemma~\ref{prop:cmi-aug}}

\begin{proof}
By the definition of conditional mutual information: 
\begin{align*}
    &~I(Z'; Z \mid Y) - I(Z'; Z\mid A, Y)\\
    =&~ \big( H(Z\mid Y) - H (Z\mid Z', Y) \big) \\
    &~ - \big( H(Z\mid A, Y) - H (Z\mid Z', A, Y) \big) \\
    =&~\big(H(Z\mid Y) - H(Z\mid A, Y) \big) \\
    &~ + \big( H (Z\mid Z', A, Y) - H (Z\mid Z', Y) \big)  \\
    =&~I(Z; A\mid Y) \\
    & \hspace{0.5cm} + \big( H (Z\mid Z', A, Y) - H (Z\mid Z', Y) \big)  \\
    \leq&~I(Z; A\mid Y) + H (Z\mid Z', A, Y)  \\
    \leq&~I(Z; A\mid Y) + H (Z\mid Z',Y)  \\
    \leq&~I(Z; A\mid Y) + \epsilon.
\end{align*}

Next, we prove the opposite side,  
\begin{equation}
    \nonumber
    \begin{aligned}
    &~I(Z'; Z \mid Y) - I(Z'; Z\mid A, Y)\\
    =&~ \big( H(Z\mid Y) - H (Z\mid Z', Y) \big) \\
    &~ - \big( H(Z\mid A, Y) - H (Z\mid Z', A, Y) \big) \\
    =&~\big(H(Z\mid Y) - H(Z\mid A, Y) \big) \\
    &~ + \big( H (Z\mid Z', A, Y) - H (Z\mid Z', Y) \big)  \\
    =&~I(Z; A\mid Y) \\
    & \hspace{0.5cm} + \big( H (Z\mid Z', A, Y) - H (Z\mid Z', Y) \big)  \\
    \geq&~I(Z; A\mid Y) - H (Z\mid Z', Y) \\
    \geq&~I(Z; A\mid Y) - \epsilon,
    \end{aligned}
\end{equation}
which completes the proof.
\end{proof}

\subsection{Proof of Proposition~\ref{thm:cmi-sup-constrative-upper}}

\begin{proof}
\begin{equation}
    \nonumber
    \begin{aligned}
        I(& Z'; Z \mid Y) =  - H(Z' \mid Z, Y) + H(Z' \mid Y) \\
        =&~ \mathbb{E}_{p(y)} \big[ \mathbb{E}_{p(z', z\mid y)}[\log p(z'\mid z, y)] \\
        &~ - \mathbb{E}_{p(z'\mid y)}[\log p(z'\mid y)] \big]\\
        =&~ \mathbb{E}_{p(y)} \big[ \mathbb{E}_{p(z', z\mid y)}[\log p(z'\mid z, y)] \\
        &~ - \mathbb{E}_{p(z'\mid y)}[ \log \mathbb{E}_{p(z\mid y)}[ p(z'\mid z, y)] ] \big] \\
        \leq &~ \mathbb{E}_{p(y)} \big[ \mathbb{E}_{p(z', z\mid y)}[\log p(z'\mid z, y)] \\
        &~ - \mathbb{E}_{p(z'\mid y)}[\mathbb{E}_{p(z\mid y)}[\log p(z'\mid z, y)]] \big] \\
        \leq &~ - \mathbb{E}_{p(y)} \big[ \mathbb{E}_{p(z'\mid y)}[\mathbb{E}_{p(z\mid y)}[\log p(z'\mid z, y)]] \big] \\
    \end{aligned}
\end{equation}
where the second line follows the marginal of a joint distribution can be expressed as the expectation of the corresponding conditional distribution, and the third line follows Jensen's Inequality.

\end{proof}

\subsection{Proof of Proposition~\ref{thm:cmi-sup-constrative-lower}}

The proof techniques used in Proposition 3.2 follow in Proposition 2.4 in~\citet{tsai2021conditional}, which could be dated back to~\citet{oord2018representation, pmlr-v97-poole19a}. To make the paper self-contained, we include all the details of the lemmas to get the final results.

The proof of Proposition~\ref{thm:cmi-sup-constrative-lower} is dependent on the Lemmas~\ref{lemma:nguyen}-\ref{lemma:CMI} showed in Figures~\ref{fig:appendix-A-lemmas-1}~and~\ref{fig:appendix-A-lemmas-2} as well as Proposition~\ref{prop:weakCMI-CMI} in Figure~\ref{fig:proof-weakCMI-CMI}. Finally, we present the proof of Proposition~\ref{thm:cmi-sup-constrative-lower} in~Figure~\ref{fig:proof-thm-cmi-sup-contrastive-lower}.

\begin{figure*}
\begin{minipage}{\textwidth}
\begin{framed}

\begin{lemma}[\citep{Nguyen2010estimating}] Let $\zspace$ be the sample space for $Z'$ and $Z$, $s:\zspace\times\zspace \to \mathbb{R}$ be any function, and $\mathcal{P}$ and $\mathcal{Q}$ be the probability measures over $\zspace\times\zspace$. We have
\begin{equation}
    \nonumber
    \dkl(\mathcal{P} \| \mathcal{Q}) = \sup_{s} \mathbb{E}_{(z',z)\sim \mathcal{P}} [s(z',z)] - \mathbb{E}_{(z',z)\sim \mathcal{Q}} [\exp(s(z',z))] + 1
\end{equation}
\label{lemma:nguyen}
\end{lemma}

\begin{proof}
We first get the second-order functional derivative of the objective: $- \exp(s(z',z)) \cdot d \mathcal{Q}$, which is negative and it implies there is a supreme value for the objective. Next, we set the first-order functional derivative of the objective to be zero:
\begin{equation}
    \nonumber
    d \mathcal{P} - \exp(s(z',z)) \cdot d \mathcal{Q} = 0. 
\end{equation}
Reorganizing the equation above we get the optimal similarity function $s^*(z',z) = \log(\frac{d \mathcal{P}}{d \mathcal{Q}})$. Plugging it into the original objective, we have
\begin{align*}
    &~ \mathbb{E}_{\mathcal{P}} [s^*(z',z)] - \mathbb{E}_{\mathcal{Q}}[\exp(s^*(z',z))] + 1  = \mathbb{E}_{\mathcal{P}} [\log(\frac{d \mathcal{P}}{d \mathcal{Q}})] = \dkl(\mathcal{P} \| \mathcal{Q}).
\end{align*}
\end{proof}

\begin{lemma}[Four-variable variant of Lemma~\ref{lemma:nguyen}]
Let $\zspace$ be the sample space for $Z'$ and $Z$, $\yspace$ be the sample space for $Y$, $\aspace$ be the sample space for $A$, $s:\zspace\times\zspace\times\yspace\times\aspace \to \mathbb{R}$ be any function, and $\mathcal{P}$ and $\mathcal{Q}$ be the probability measures over $\zspace\times\zspace\times\yspace\times\aspace$. We have
\begin{equation}
    \nonumber
    \dkl(\mathcal{P} \| \mathcal{Q}) = \sup_{s} \mathbb{E}_{(z',z, y, a)\sim \mathcal{P}} [s(z',z, y, a)] - \mathbb{E}_{(z',z, y, a)\sim \mathcal{Q}} [\exp(s(z',z, y, a))] + 1
\end{equation}
\label{lemma:nguyen-four}
\end{lemma}

\begin{proof}
    The proof technique is identical to the proof of Lemma~\ref{lemma:nguyen} and the only difference is that the similarity function takes four variables as input.
\end{proof}

\begin{lemma}
$\sup_{s} \mathbb{E}_{ (z', z_1) \sim \mathcal{P}, (z', z_{2:N}) \sim \mathcal{Q}^{\otimes N-1} } \big[ \log \frac{\exp(s(z', z_1))}{\frac{1}{N} \sum_{j=1}^N \exp(s(z', z_j)) } \big] \leq \dkl(\mathcal{P} \| \mathcal{Q})$
\label{lemma:kl-lower}
\end{lemma}

\begin{proof}
    \begin{equation}
        \nonumber
        \begin{aligned}
        \dkl(& \mathcal{P} \| \mathcal{Q}) = \mathbb{E}_{(z', z_{2:N}) \sim \mathcal{Q}^{\otimes N-1} } \bigg[ \dkl(\mathcal{P} \| \mathcal{Q}) \bigg] \\
        &\geq  \mathbb{E}_{(z', z_{2:N}) \sim \mathcal{Q}^{\otimes N-1} } \bigg[ \mathbb{E}_{\mathcal{P}} [\log \frac{\exp(s^*(z',z))}{\frac{1}{N} \sum_{j=1}^N \exp(s(z', z_j)) }] - \mathbb{E}_{\mathcal{Q}}[\frac{\exp(s^*(z',z))}{\frac{1}{N} \sum_{j=1}^N \exp(s(z', z_j))}] + 1 \bigg] \\
        &=  \mathbb{E}_{(z', z_{2:N}) \sim \mathcal{Q}^{\otimes N-1} } \bigg[ \mathbb{E}_{\mathcal{P}} [\log \frac{\exp(s^*(z',z))}{\frac{1}{N} \sum_{j=1}^N \exp(s(z', z_j)) }] - 1 + 1 \bigg] \\
        &= \mathbb{E}_{ (z', z_1) \sim \mathcal{P}, (z', z_{2:N}) \sim \mathcal{Q}^{\otimes N-1} } \big[ \log \frac{\exp(s(z', z_1))}{\frac{1}{N} \sum_{j=1}^N \exp(s(z', z_j)) } \big],
        \end{aligned}
    \end{equation}
where the first line follows the fact that $\dkl(\mathcal{P} \| \mathcal{Q})$ is a constant, the second line follows Lemma~\ref{lemma:nguyen}, the third line follows the fact that $(z', z_1)$ and $(z', z_{2:N})$ are interchangeable when sampling from $\mathcal{Q}$. Thus, for any similarity function $s$, we have
\begin{equation}
    \nonumber
    \sup_{s} \mathbb{E}_{ (z', z_1) \sim \mathcal{P}, (z', z_{2:N}) \sim \mathcal{Q}^{\otimes N-1} } \big[ \log \frac{\exp(s(z', z_1))}{\frac{1}{N} \sum_{j=1}^N \exp(s(z', z_j)) } \big] \leq \dkl(\mathcal{P} \| \mathcal{Q})
\end{equation}
\end{proof}

\end{framed}
\end{minipage}
\caption{Lemmas required for the proof of Proposition~\ref{thm:cmi-sup-constrative-lower}.}
\label{fig:appendix-A-lemmas-1}
\end{figure*}
\begin{figure*}
\begin{minipage}{\textwidth}
\begin{framed}

\begin{lemma}
\begin{equation}
\nonumber
\begin{aligned}
&~\dkl( P_{Z', Z}~\|~ \mathbb{E}_{P_{A,Y}}[P_{Z'\mid A,Y} P_{Z\mid A,Y}] )\\ 
=&~\sup_{s} \mathbb{E}_{ (z', z) \sim P_{Z', Z}}[s(z', z)] - \mathbb{E}_{ (z', z) \sim \mathbb{E}_{P_{A,Y}}[P_{Z'\mid A,Y} P_{Z\mid A,Y}]}[\exp(s(z', z))] + 1.
\end{aligned}
\end{equation}
\label{lemma:weak-CMI}
\end{lemma}

\begin{proof}
   We use Lemma~\ref{lemma:nguyen} and substitute $\mathcal{P}$ and $\mathcal{Q}$ with $P_{Z', Z}$ and $\mathbb{E}_{P_{A,Y}}[P_{Z'\mid A,Y} P_{Z\mid A,Y}]$, respectively.
\end{proof}

\begin{lemma}
\begin{equation}
\nonumber
\begin{aligned}
&\, I(Z'; Z  \mid A, Y) \\
=&~\dkl ( P_{Z', Z, A, Y}~\|~ P_{A, Y} P_{Z'\mid A, Y}P_{Z\mid A, Y}) \\
=&~\sup_{s} \mathbb{E}_{ (z', z, a, y) \sim P_{Z', Z, A, Y}}[s(z', z, a, y)] \\
& \hspace{1cm} - \mathbb{E}_{(z', z, a, y) \sim P_{A, Y} P_{Z'\mid A, Y}P_{Z\mid A, Y}  }[\exp(s(z', z, a, y))] + 1.
\end{aligned}
\end{equation}
\label{lemma:CMI}
\end{lemma}

\begin{proof}
    We use Lemma~\ref{lemma:nguyen-four} and substitute $\mathcal{P}$ and $\mathcal{Q}$ with $P_{Z', Z, A, Y}$ and $P_{A, Y} P_{Z'\mid A, Y}P_{Z\mid A, Y}$, respectively.
\end{proof}

\end{framed}
\end{minipage}
\caption{Lemmas required for the proof of Proposition~\ref{thm:cmi-sup-constrative-lower}.}
\label{fig:appendix-A-lemmas-2}
\end{figure*}
\begin{figure*}
\begin{minipage}{\textwidth}
\begin{framed}

\begin{proposition} \label{prop:weakCMI-CMI}
\begin{align*}
\dkl( & P_{Z', Z}~\|~\mathbb{E}_{P_{A,Y}}[P_{Z'\mid A,Y} P_{Z\mid A,Y}] ) \leq I(Z'; Z \mid A, Y)
\end{align*}
\end{proposition}

\begin{proof}
We have

\begin{align*}
    \nonumber
    &~\dkl( P_{Z', Z}~\|~ \mathbb{E}_{P_{A,Y}}[P_{Z'\mid A,Y} P_{Z\mid A,Y}] ) \\
    =&~\sup_{s} \mathbb{E}_{ (z', z) \sim P_{Z', Z}}[s(z', z)] - \mathbb{E}_{ (z', z) \sim \mathbb{E}_{P_{A,Y}}[P_{Z'\mid A,Y} P_{Z\mid A,Y}]}[\exp(s(z', z))] + 1 \\
    =&~ \sup_{s} \mathbb{E}_{ (z', z, a, y) \sim P_{Z', Z, A, Y}}[s(z', z)] - \mathbb{E}_{ (z', z, a, y) \sim P_{A,Y} P_{Z'\mid A,Y} P_{Z\mid A,Y}}[\exp(s(z', z))] + 1, \\
\end{align*}

where the first equation follows Lemma~\ref{lemma:weak-CMI}. Let $s^*(z', z)$ be the function when the supreme value is achieved and let $\hat{s}^*(z', z, a, y) = s^*(z', z), \forall~(a, y)\in P_{A, Y}$, and we have
\begin{equation}
    \nonumber
    \begin{aligned}
    &~\dkl( P_{Z', Z}~\|~ \mathbb{E}_{P_{A,Y}}[P_{Z'\mid A,Y} P_{Z\mid A,Y}] ) \\
    =&~ \sup_{s} \mathbb{E}_{ (z', z, a, y) \sim P_{Z', Z, A, Y}}[s(z', z)] - \mathbb{E}_{ (z', z, a, y) \sim P_{A,Y} P_{Z'\mid A,Y} P_{Z\mid A,Y}}[\exp(s(z', z))] + 1 \\
    =&~ \mathbb{E}_{ (z', z, a, y) \sim P_{Z', Z, A, Y}}[\hat{s}^*(z', z, a, y)] - \mathbb{E}_{ (z', z, a, y) \sim P_{A,Y} P_{Z'\mid A,Y} P_{Z\mid A,Y}}[\exp(\hat{s}^*(z', z, a, y)] + 1 \\
    \leq&~ \sup_{\hat{s}} \mathbb{E}_{ (z', z, a, y) \sim P_{Z', Z, A, Y}}[\hat{s}(z', z, a, y)] - \mathbb{E}_{ (z', z, a, y) \sim P_{A,Y} P_{Z'\mid A,Y} P_{Z\mid A,Y}}[\exp(\hat{s}(z', z, a, y)] + 1 \\
    =&~I(Z'; Z \mid A, Y), \\
    \end{aligned}
\end{equation}
where the last equation follows Lemma~\ref{lemma:CMI}.
\end{proof}
\end{framed}
\end{minipage}
\caption{Proposition~\ref{prop:weakCMI-CMI} and its proof.}
\label{fig:proof-weakCMI-CMI}
\end{figure*}

\begin{figure*}
\begin{minipage}{\textwidth}
\begin{framed}
\begin{proof}

Define two probability measures $\mathcal{P} = P_{Z', Z}$ and $\mathcal{Q} = \mathbb{E}_{P_{A,Y}}[P_{Z'\mid A,Y} P_{Z\mid A,Y}] $, we have 

\begin{equation}
    \nonumber
    \begin{aligned}
    &\,\mathbb{E}_{p(a, y)} \bigg[  \mathbb{E}_{p(z'_i, z_i \mid a, y)^{\otimes N}} \big[ \log \frac{\exp(s(z'_i, z_i))}{\frac{1}{N} \sum_{j=1}^N \exp(s(z'_i, z_j)) } \big] \bigg] \\
    =&~\mathbb{E}_{ (z', z_1) \sim \mathcal{P}, (z', z_{2:N}) \sim \mathcal{Q}^{\otimes N-1} } \big[ \log \frac{\exp(s(z', z_1))}{\frac{1}{N} \sum_{j=1}^N \exp(s(z', z_j)) } \big] \\
    \leq&~\dkl(\mathcal{P}  \| \mathcal{Q}) \\
    =&~ \dkl( P_{Z', Z}~\|~ \mathbb{E}_{P_{A,Y}}[P_{Z'\mid A,Y} P_{Z\mid A,Y}] )  \\
    \leq&~ I(Z'; Z \mid A, Y).  \\
    \end{aligned}
\end{equation}
where the second equation follows Lemma~\ref{lemma:kl-lower} and the last equation follows Proposition~\ref{prop:weakCMI-CMI}. 

\end{proof}
\end{framed}
\end{minipage}
\caption{Proof of Proposition~\ref{thm:cmi-sup-constrative-lower}.}
\label{fig:proof-thm-cmi-sup-contrastive-lower}
\end{figure*}

\section{Experimental Details}
\label{app:exp-details}

\subsection{Data prepossessing pipelines}

\paragraph{Jigsaw.} Our first dataset, which we refer to as \texttt{jigsaw}, is a corpus of comments from an online forum associated with a toxicity rating. 
\texttt{jigsaw}'s main task is binary classification: given a ``toxicity'' score in the range \([0, 1]\) that has been assigned to each comment, we determine whether the ``toxicity'' score is greater or equal to 0.5. 
Each comment is also annotated with some ``identity'' labels, indicating whether some identities belonging to specific demographic groups are mentioned in the comment. We focus on the identity labels related to ``race or ethnicity'' and binaries the identity labels into black and non-black.
Note that there are other sensitive attributes in the \textit{Jigsaw-Toxicity} dataset, and we constrain the scope of our study to the ``race'' attributes present in text classification datasets.
We follow~\citep{wilds2021} to perform the train/val/test splits. 
The data with ``race or ethnicity'' identity labels are split into training, validation, and test sets, summarized in Table~\ref{table:appendix-B:jigsaw-split-summary}. 

\begin{table*}[h]
    \centering
    \caption{Summary of training, validation, and test splits for the \texttt{jigsaw} dataset.}
    \begin{tabular}{ p{2cm} | p{3cm} | p{4cm} | c }
    \hline
    \textbf{Data split} & \textbf{Samples} & \textbf{Protected attribute} \newline (NB = non-black, B = black) & \textbf{Task label average} \\ \hline \hline
    Training & 25,954 (60.5\%) & 61.9\% (NB), 38.1\% (B) & 0.2822 \\ \hline
    Validation & 4,390 (10.2\%) & 62.4\% (NB), 37.6\% (B) & 0.2897 \\ \hline
    Test & 12,562 (29.3\%) & 61.2\% (NB), 38.8\% (B) & 0.2873 \\ \hline \hline
    \textbf{Total} & 42,906 & 61.7\% (NB), 38.3\% (B) & 0.2844 \\ \hline
    \end{tabular}
    \label{table:appendix-B:jigsaw-split-summary}
\end{table*}

\paragraph{Bias-in-Bios.} To measure model fairness and performance in the multi-class classification setting, we use the professional biographies dataset of \cite{de2019bias}, which we refer to as the \texttt{biasbios} dataset. The data consist of nearly 400,000 online biographies collected from the Common Crawl corpus. These biographies are annotated with one of the 28 professions to which their subject belongs. The data are mapped to a binary gender based on the occurrence of gendered pronouns and are scrubbed to exclude the authors' names and pronouns. It is worth noting that mapping gender to binary labels is a strong simplified assumption to map data to a demographic label cleanly; it ignores people who do not identify as female or male, as well as the complexity of gender identity more generally. We refer readers to the original work~\citep{de2019bias} for further discussion of these issues. For our experiments, we attempt to predict the profession as our task label while protecting against the gender attribute. We replicate the splits of \texttt{biasbios} used by \cite{ravfogel-etal-2020-null}, which are summarized in Table \ref{table:appendix-B:biasbios-split-summary}. 

\begin{table*}
    \centering
    \caption{Summary of the training, validation, and test splits of the \texttt{biasbios} dataset.}
    \begin{tabular}{ p{2cm} | p{3cm} | p{4cm} }
    \hline
    \textbf{Data split} & \textbf{Samples} & \textbf{Protected attribute} \newline (F = female, M = male) \\ \hline \hline
    Training & 255,710 (65.0\%) & 46.0\% (F), 54.0\% (M) \\ \hline
    Validation & 39,369 (10.0\%) & 47.8\% (F), 52.2\% (M) \\ \hline
    Test & 98,344 (25.0\%) & 46.5\% (F), 53.5\% (M) \\ \hline \hline
    \textbf{Total} & 393,423 & 46.3\% (F), 53.7\% (M) \\ \hline
    \end{tabular}
    \label{table:appendix-B:biasbios-split-summary}
\end{table*}


\subsection{Detailed Implementations and Hyperparameter Settings}

In this section, we provide more details on our implementations and give the hyperparameters we use in our experiments. We first detail how we tune each method's performance and fairness trade-offs.

\begin{itemize}
    \item \textbf{One-stage / Two-stage CL.} for one- and two-stage CL, once we determine the best classification performance by conducting a grid search on temperature, (pre-training) batch size, and data augmentation strategies (as well as $\gamma$ in one-stage CL) , we only tune the parameter \(\lambda\) described in Sec.~\ref{subsec:methodology:practical-implementation}, which affects the trade-offs between supervised contrastive loss \(L_{\text{sup}}\) and the conditional supervised InfoNCE loss \(L_{\text{CS-InfoNCE}}\). For two-stage CL, we set the pre-training epochs to be 15 and 25 for \texttt{jigsaw} and \texttt{biasbios}, respectively, and early stop the pre-training if there is no improvement on the validation set for three consecutive epochs.

    \item \textbf{Diverse adversarial training.} Following~\citet{han-etal-2021-diverse}, we use an ensemble of three adversarial discriminators and the same adversarial network architecture. There are two hyperparameters of interest: \(\lambda_{diff}\) and \(\lambda_{adv}\). \(\lambda_{diff}\) is a difference loss hyperparameter that encourages discriminators to learn orthogonal representations. \(\lambda_{adv}\) affects the trade-offs between task performance and fairness. We first do a grid search on $\lambda_{diff}=\{0, 100, 1000, 5000\}$ and vary the values of \(\lambda_{adv}\) to determine the best hyperparameter configurations.
    
    \item \textbf{Adversarial training.} The implementation is nearly identical to diverse adversarial training, except that there is just one adversarial discriminator. 
    
    \item \textbf{INLP.} Following~\citet{ravfogel-etal-2020-null}, we use the weights of an SVM classifier as the parameters of the linear guarding layer and follow the same hyperparameters in training the linear guarding layer. The trade-off hyperparameter that we tune for INLP is \(N_{clf}\), which is the number of classifiers trained by INLP (\ie, the number of rounds).
\end{itemize}

Table~\ref{table:appendix-B:jigsaw-method-hyperparameters} contains the trade-off hyperparameters used for our experiments on the \texttt{jigsaw} dataset, while Table~\ref{table:appendix-B:biasbios-method-hyperparameters} summarizes the trade-off hyperparameter choices for the \texttt{biasbios} dataset. The remaining hyperparamters for all methods are listed in Table~\ref{table:appendix-B:additional-hyperparameters}.

\begin{table*}[h!]
    \centering
    \caption{Trade-off hyperparameters tested for RQ2 (Figure \ref{fig:results:rq2-fair-perf}) for the \texttt{jigsaw} dataset.}
    \label{table:appendix-B:jigsaw-method-hyperparameters}
    \begin{tabular}{ p{5cm} | p{7cm} }
    \hline
    \textbf{Method} & \textbf{Hyperparameters tested} \\ \hline \hline
    Adversarial training & \(\lambda_{adv} \in \{0.1, 0.5, 1, 2\}\) \\ \hline
    Diverse adversarial training \newline (\(N_{adv} = 3\), \(\lambda_{diff} = 100\)) & \(\lambda_{adv} \in \{0.1, 0.5, 1, 2\}\) \\ \hline
    INLP & \(N_{clf} \in \{20, 50, 80, 100, 150\}\) \\ \hline
    One-stage CL & \(\lambda \in \{0, 1, 2, 5\}\) \\ \hline
    Two-stage CL & \(\lambda \in \{0, 0.1, 0.5, 1, 2, 5\}\) \\ \hline
    \end{tabular}
\end{table*}

\begin{table*}[h!]
    \centering
    \caption{Trade-off hyperparameters tested for RQ2 (Figure \ref{fig:results:rq2-fair-perf}) for the \texttt{biasbios} dataset.}
    \begin{tabular}{ p{5cm} | p{7cm} }
    \hline
    \textbf{Method} & \textbf{Hyperparameters tested} \\ \hline \hline
    Adversarial training & \(\lambda_{adv} \in \{0.1, 0.2, 0.5, 1\}\) \\ \hline
    Diverse adversarial training \newline (\(\lambda_{diff} = 5000\)) & \(\lambda_{adv} \in \{0.1, 0.2, 0.5, 1\}\) \\ \hline
    INLP & \(N_{clf} \in \{20, 50, 100, 300, 400\}\) \\ \hline
    One-stage CL & \(\lambda \in \{0, 1, 5, 10\}\) \\ \hline
    Two-stage CL & \(\lambda \in \{0, 1, 5, 10\}\) \\ \hline
    \end{tabular}
    \label{table:appendix-B:biasbios-method-hyperparameters}
\end{table*}

\begin{table*}[h!]
    \centering
    \caption{Additional hyperparameters used for experiments.}
    \begin{tabular}{ c |
    >{\centering}p{0.30\textwidth} |
    >{\centering\arraybackslash}p{0.30\textwidth} }
    \hline
    \textbf{Hyperparameter} & (\texttt{jigsaw}) & (\texttt{biasbios}) \\ \hline \hline
    Batch size (fine-tuning) & 32 & 32 \\ \hline
    Learning rate & 2e-5 & 2e-5 \\ \hline
    Epochs & 10 & 7 \\ \hline
    Optimizer & Adam & Adam \\ \hline
    \end{tabular}
    \label{table:appendix-B:additional-hyperparameters}
\end{table*}

\subsection{Data Augmentation Strategies}

In this section, we provide a description of the data augmentation strategies used in CL-based methods\footnote{We use the implementations of data augmentation at: \url{https://github.com/makcedward/nlpaug}.}.
\begin{itemize}
    \item \textbf{Easy data augmentation (EDA)\citep{wei-zou-2019-eda}.} EDA consists of four simple operations: synonym replacement, random insertion, random swap, and random deletion. Following the suggestions provided by the original paper, we choose the augmentation ratio to be 0.1 and create four augmented examples per example. 
    \item \textbf{Back translation~\citep{edunov-etal-2018-understanding}.} It first translates the input example to another language and back to English. We use the machine translation model \texttt{wmt19-en-de} in our experiment. 
    \item \textbf{Word replacement using contextual language model (CLM insert)~\citep{kobayashi-2018-contextual}.} It replaces words based on a language model that leverages contextual word embeddings to find the most similar word for augmentation. We use the RoBERTa-base language model and choose the augmentation rate of 0.1.
    \item \textbf{Word insertion using contextual language model (CLM insert)~\citep{kobayashi-2018-contextual}.} It inserts words based on a language model that leverages contextual word embeddings to find the most similar word for augmentation. We use the RoBERTa-base language model and choose the augmentation rate of 0.1.
\end{itemize}

\section{Additional Experimental Results}
\label{app:add-exp-res}

\subsection{More Comments for CL-based Methods}
\label{app-subsec:cl}


\begin{figure*}
    \begin{subfigure}[!thb]{\textwidth}
    \centering
    \includegraphics[width=1\textwidth]{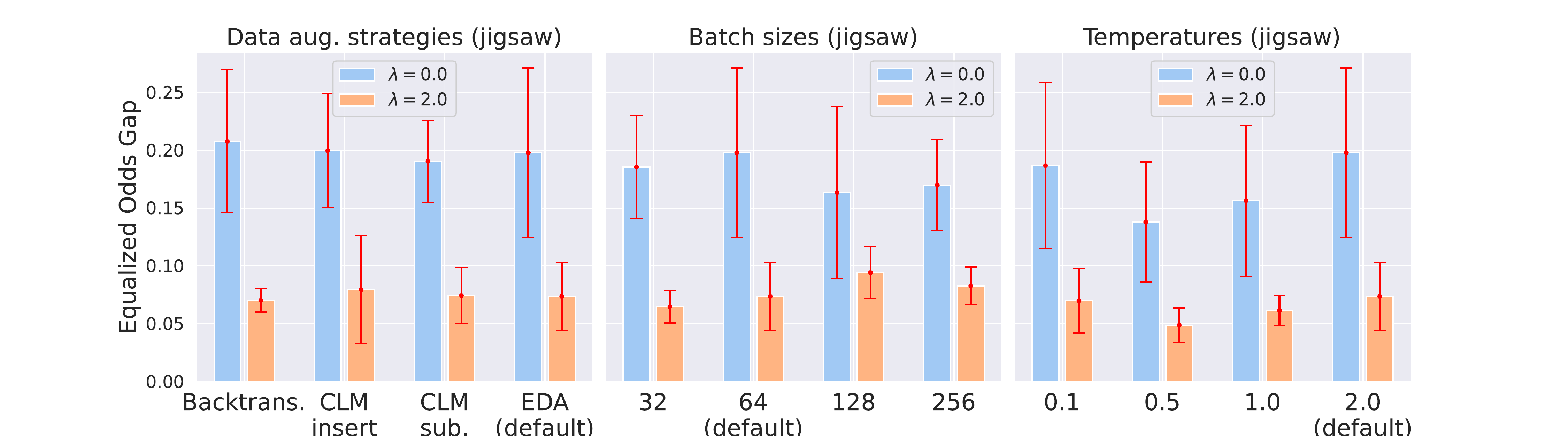}
    \end{subfigure} \begin{subfigure}[b]{\textwidth}
    \centering
    \includegraphics[width=\textwidth]{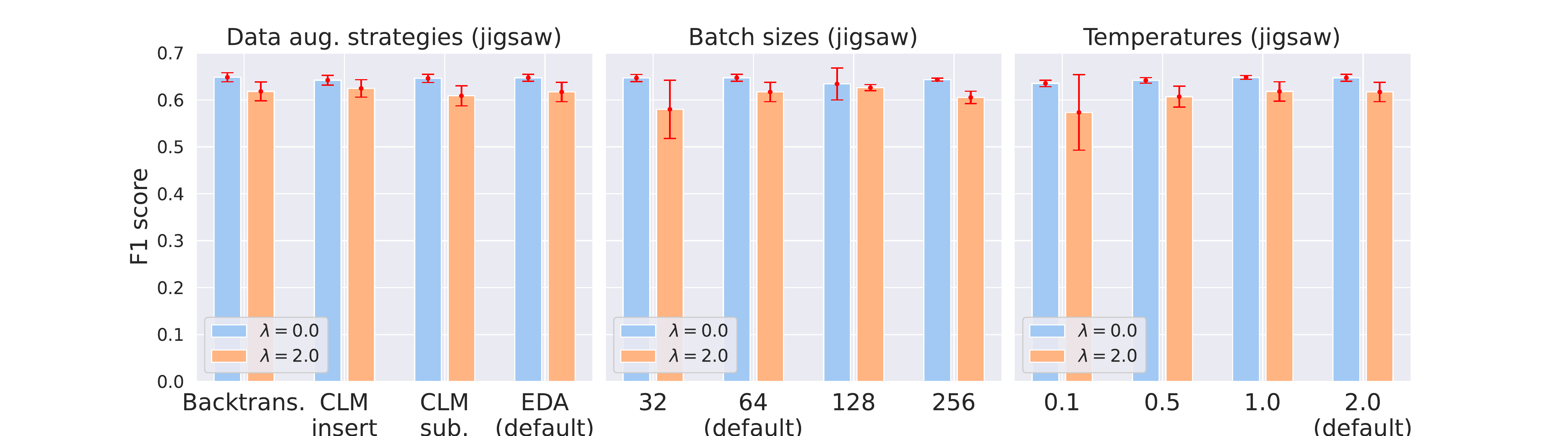}
    \end{subfigure}
    \caption{Sensitivity analysis of two-stage CL to key hyperparameter changes in (\texttt{Jigsaw}).}
    \label{fig:rq3-jigsaw}
\end{figure*}

Our method achieves highly consistent results w.r.t. fairness and performance compared to the baseline methods. 
Figure~\ref{fig:rq3-jigsaw} visualizes the model performance and EO gaps of two-stage CL under different hyperparameter settings when $ \lambda \in \{0.0, 2.0\}$ in the \texttt{jigsaw} dataset.

\subsection{Visualization of the BERT Embeddings using Different Objectives}
\label{app-subsec:tsne-visualization}

\begin{figure*}
     \centering
     \begin{subfigure}[b]{0.48\textwidth}
         \centering
         \includegraphics[width=\textwidth]{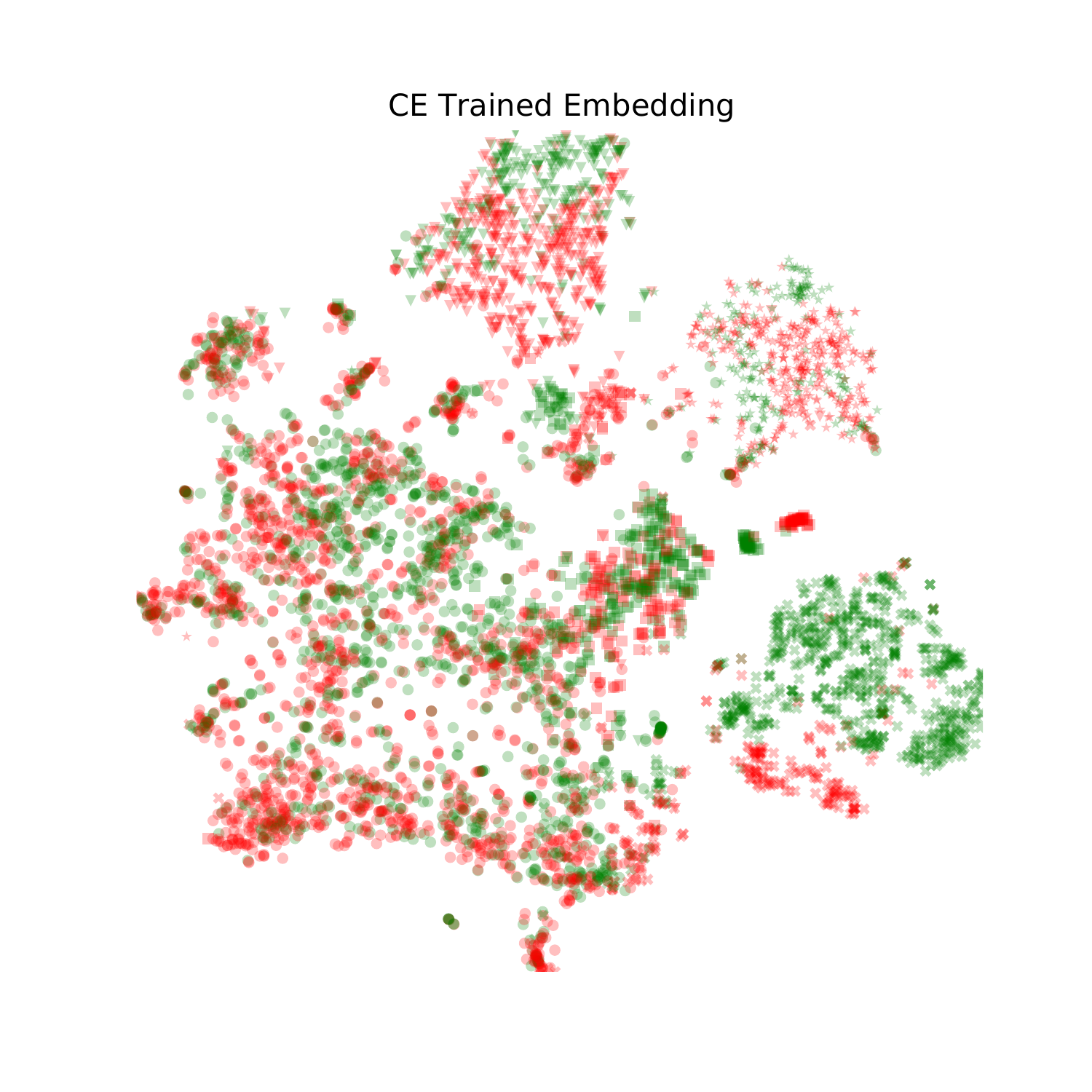}
     \end{subfigure}
     \hfill
     \begin{subfigure}[b]{0.48\textwidth}
         \centering
         \includegraphics[width=\textwidth]{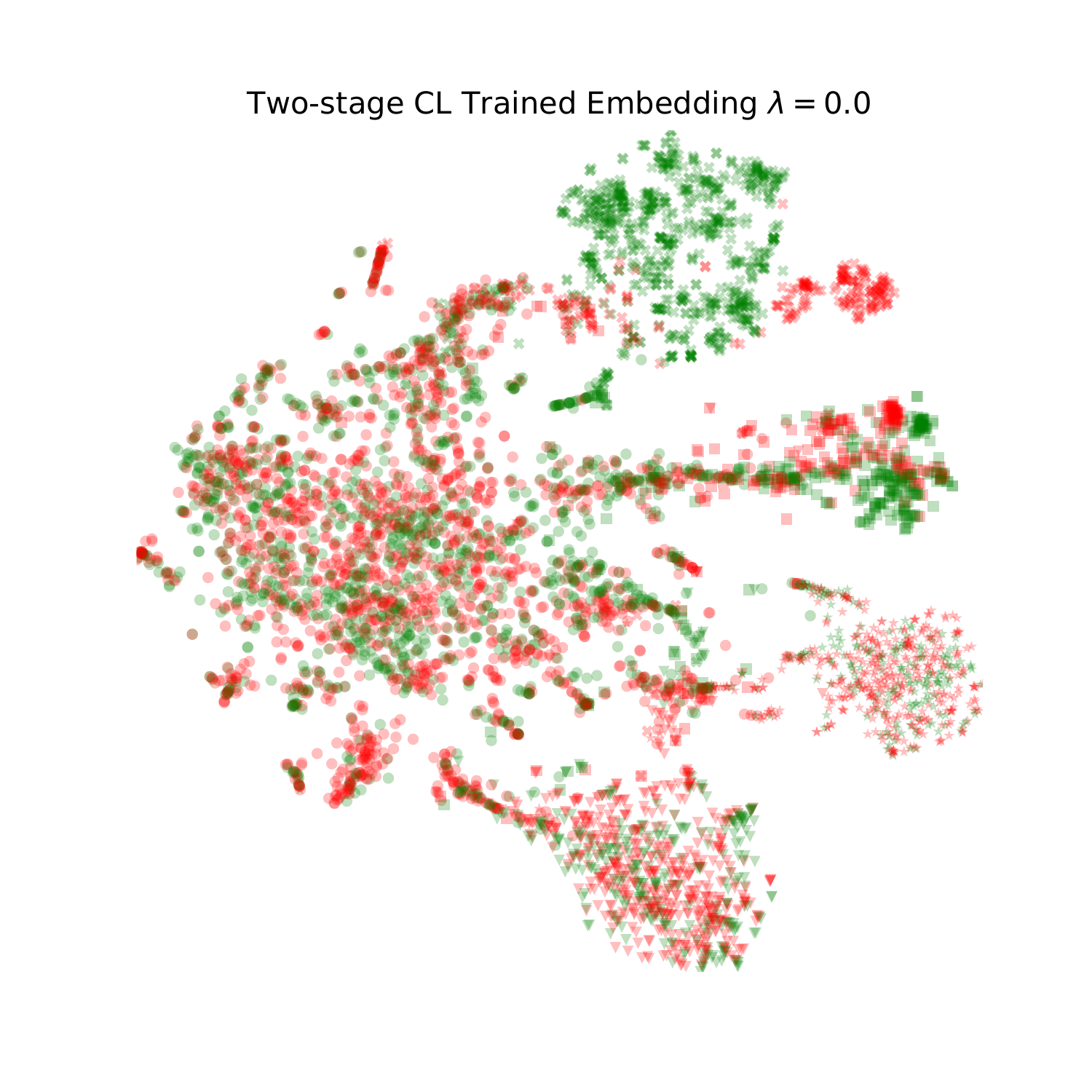}
     \end{subfigure}
     \begin{subfigure}[b]{0.48\textwidth}
         \centering
         \includegraphics[width=\textwidth]{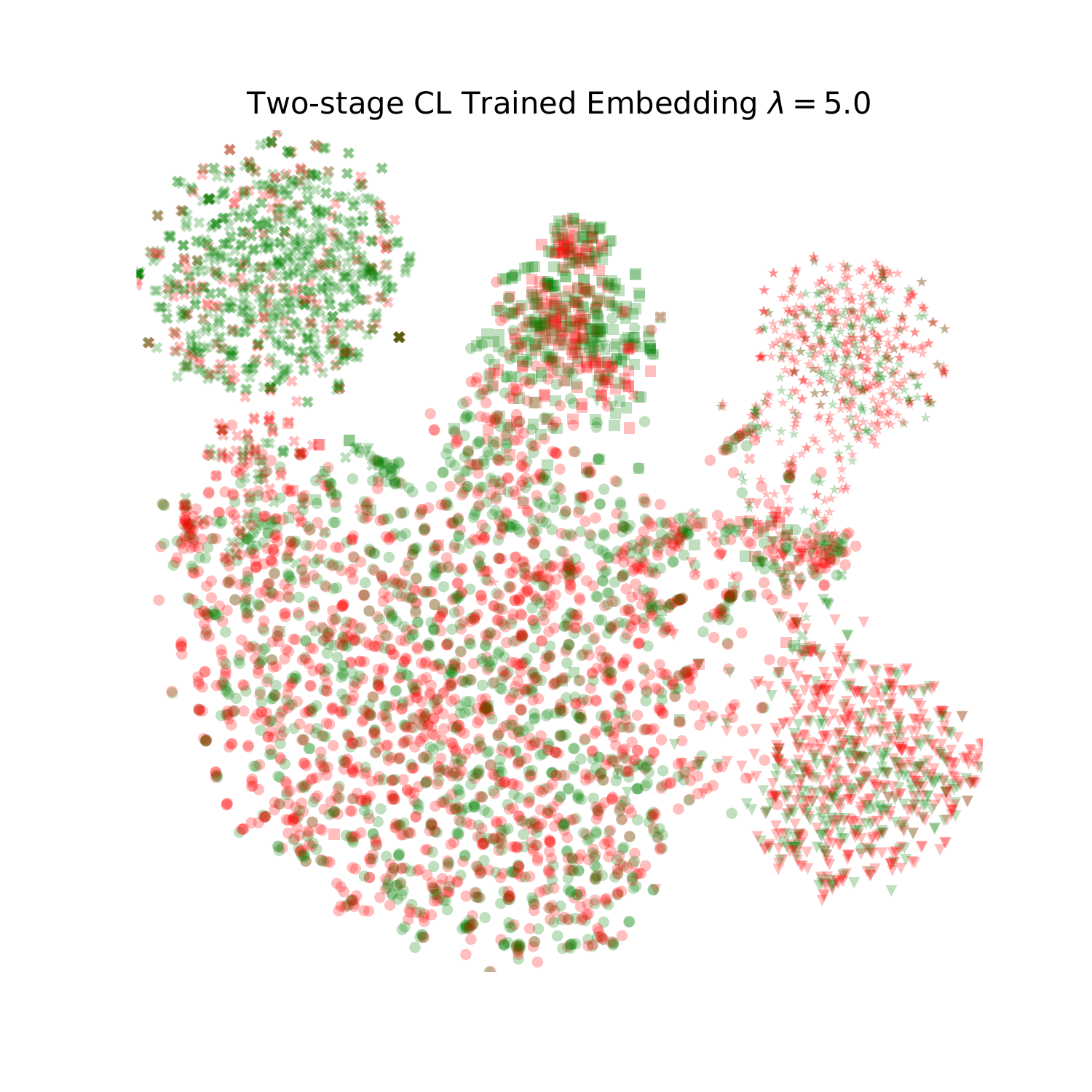}
     \end{subfigure}
     \caption{T-SNE visualization of text embeddings using different training objectives (zoom in for better visualization) in the \texttt{biasbios} dataset. Different colors indicate different sensitive attributes (e.g., red for males and green for females), and different markers indicate different classes. CE-trained and CL-trained embedding capture the class information well (points with the same markers form their own clusters). However, points with the same sensitive attributes within the same class are more likely to form small clusters. When we introduce $L_{\text{CS-InfoNCE}}$, those points tend to be more aligned.}
     \label{fig:tsne}
\end{figure*}

In Figure~\ref{fig:tsne}, we show the T-SNE visualization~\citep{van2008visualizing} of text embeddings learned with different training objectives. 
We can see that both CE-trained and CL-trained embeddings capture the class information well (points with the same markers form their own clusters). However, points with the same sensitive attributes (the same colors) within the same class are more likely to form small clusters. When we introduce $L_{\text{CS-InfoNCE}}$, those points tend to be more aligned.

\subsection{How Different Pretrained Text Encoders affects the Performance of INLP?}
\label{app-subsec:inlp}

To provide the clearest comparison between our proposed methods and the baselines, we used the best settings for the baseline methods we could attain. Nonetheless, we observed that the performance of INLP was highly sensitive to the encoder training settings, which could be an important practical consideration for practitioners selecting between different ways of improving model fairness. Figure~\ref{fig:inlp-comparison} compares the performance and fairness of INLP using different encoder pre-training strategies. We see that in both datasets, the classification and fairness performance of INLP changes drastically even with the same values of trade-off parameter. 
Even training with the same objectives (CE loss), the text encoders obtained in different epochs after convergence greatly affect its performance. 
For example, the CE-trained encoder obtained in the last epoch of training nearly shows no effects on bias mitigation.
If we do not train the text encoder using our datasets and directly use the parameters of the \texttt{bert-base-uncased} (this is the experimental setting of the previous work~\citep{ravfogel-etal-2020-null}), the model performances drastically decrease as the training iterations of INLP increase. Lastly, INLP does not perform well when using supervised contrastive loss to train the text encoder.
In comparison, our methods are more robust to hyperparameter changes.

\begin{figure*}
     \centering
     \includegraphics[width=\textwidth]{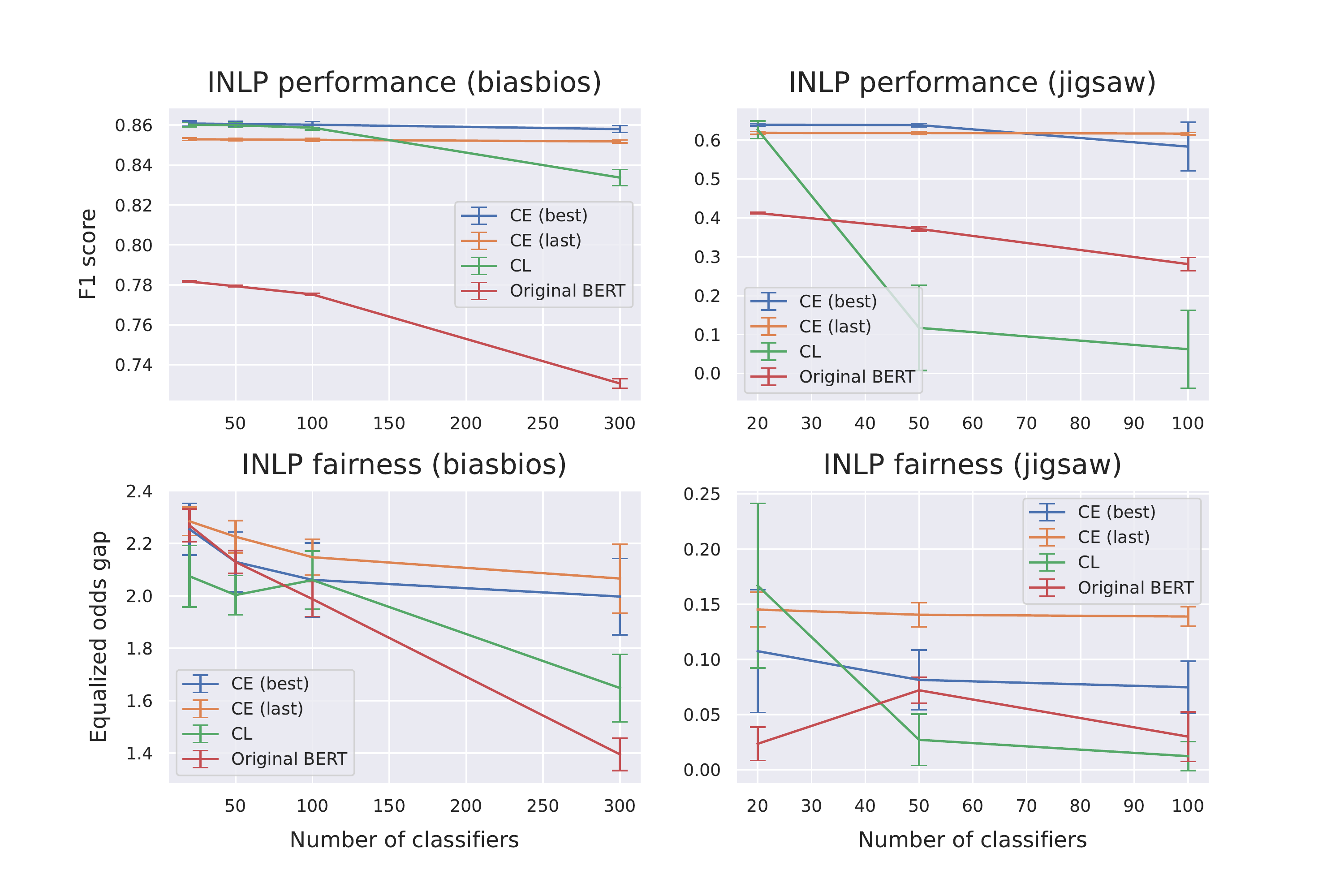}
     \caption{Comparison of INLP performance and fairness under different pretrained encoders. 
     CE (best) indicates that we train the text encoder using CE loss and save the encoder in the epoch that achieves the best validation loss for INLP.
     CE (last) indicates that we train the text encoder using CE loss and save the encoder in the last epoch for INLP.
     CL indicates that we train the text encoder using supervised contrastive loss. 
     Original BERT indicates that we use the \texttt{bert-based-uncased} as the text encoder. 
    }
     \label{fig:inlp-comparison}
\end{figure*}

\end{document}